\PassOptionsToPackage{table}{xcolor}
\documentclass[10pt]{article} 
\usepackage[preprint]{tmlr}

\usepackage{amsmath,amsfonts,bm,mathtools,amssymb,amsthm}

\def\eqref#1{equation~\ref{#1}}

\def\1{\bm{1}}

\DeclareMathAlphabet{\mathsfit}{\encodingdefault}{\sfdefault}{m}{sl}
\SetMathAlphabet{\mathsfit}{bold}{\encodingdefault}{\sfdefault}{bx}{n}

\usepackage{hyperref}
  \hypersetup{
  	colorlinks=true, 
  	bookmarksnumbered=true,
  	pdfborder={0 0 0},
  	citecolor=blue,
  	urlcolor=magenta,
  	linkcolor=blue,
  	bookmarkstype=toc
    }
\usepackage{url}

\usepackage{cleveref}
\usepackage{booktabs}
\usepackage{caption}
\usepackage{subcaption}
\usepackage{tikz}

\usepackage{enumitem}
\usepackage{listings}
\usepackage{multirow}
\usepackage{booktabs}
\usepackage{cite}
\usepackage{wrapfig}

\usepackage{siunitx}
\newcommand{\TODO}[1]{}

\definecolor{matplotcyan}{RGB}{0,255,255}
\definecolor{matplotlibgreen}{RGB}{0,128,0}
\definecolor{matplotlibpurple}{RGB}{128,0,128}
\definecolor{matplotliborange}{RGB}{255,165,0}
\newtheorem{theo}{Theorem}[section]

\newtheorem{lem}{Lemma}[section]
\newtheorem{asu}{Assumption}[section]

\newcommand\given{{\mathbin{}\mid\mathbin{}}}
\DeclarePairedDelimiterX\norm[1]\lVert\rVert{\ifblank{#1}{\:\cdot\:}{#1}}
\allowdisplaybreaks

\title{Robust Invariant Representation Learning by Distribution Extrapolation}

\author{\name Kotaro Yoshida \email yoshida.k.0253@m.isct.ac.jp \\
  \addr Department of Information and Communications Engineering\\
  Institute of Science Tokyo
  \AND
  \name Konstantinos Slavakis \email slavakis@ict.eng.isct.ac.jp \\
  \addr Department of Information and Communications Engineering\\
  Institute of Science Tokyo}

\def\month{MM}  
\def\year{YYYY} 
\def\openreview{\url{https://openreview.net/forum?id=XXXX}} 

\begin{document}
\maketitle

\begin{abstract}
  Invariant risk minimization (IRM) aims to enable out-of-distribution (OOD) generalization in deep learning by learning invariant representations. As IRM poses an inherently challenging bi-level optimization problem, most existing approaches---including IRMv1---adopt penalty-based single-level approximations. However, empirical studies consistently show that these methods often fail to outperform well-tuned empirical risk minimization (ERM), highlighting the need for more robust IRM implementations. This work theoretically identifies a key limitation common to many IRM variants: their penalty terms are highly sensitive to limited environment diversity and over-parameterization, resulting in performance degradation. To address this issue, a novel extrapolation-based framework is proposed that enhances environmental diversity by augmenting the IRM penalty through synthetic distributional shifts. Extensive experiments---ranging from synthetic setups to realistic, over-parameterized scenarios---demonstrate that the proposed method consistently outperforms state-of-the-art IRM variants, validating its effectiveness and robustness.
\end{abstract}

\section{Introduction}\label{ch:introduction}

In modern machine learning applications, the assumption that training and test data are independent
and identically distributed (i.i.d.) often fails to hold, with distribution shifts commonly
observed~\citep{quinonero2022dataset}. In such scenarios, models are required to generalize to
previously unseen data distributions---a challenge known as out-of-distribution (OOD)
generalization. Both theoretical analyses and empirical studies have shown that algorithms built on
the i.i.d.\ assumption, and usually realized via the classical framework of empirical risk
minimization (ERM)~\citep{b1}, tend to perform poorly under distribution shifts~\citep{b2}.

In the context of OOD generalization, invariant risk minimization (IRM)~\citep{b11} has attracted
significant attention. IRM is grounded in invariant representation learning, which aims to learn
only those features from the input data that are consistently useful for inference under
distribution shifts (referred to as invariant features), while discarding all other non-invariant
features (referred to as spurious ones). However, IRM models the learning process as a composition
of a feature extractor and a classifier, leading to a challenging bi-level optimization problem. As
a result, strict implementations of IRM have been found to be practically infeasible~\citep{b11}. To
address this limitation, \citep{b11} proposed an approximation called IRMv1, which reformulates the
original bi-level problem as a more tractable single-level optimization task. Despite numerous IRMv1
variants~\citep{b27, b23, pmlr-v162-zhou22e, b24, b37}, empirical studies have shown
that these approaches often fail to outperform a well-tuned ERM-based strategy~\citep{b29}.

This paper argues that the formulation of IRMv1---widely adopted as the foundation for many recent
IRM variants---leaves room for improvement. Specifically, the loss or penalty function in IRMv1's
single-level optimization task suffers from two key limitations: \textbf{(i)} it is highly dependent
on the diversity of training environments and does not reliably ensure optimality of the original
IRM’s lower-level objective; and \textbf{(ii)} in over-parameterized settings, this dependence is
further amplified, increasing the risk of overfitting to the training distribution.

To address the aforementioned limitations, and drawing inspiration from risk extrapolation
(REx)~\citep{b25}, this paper proposes applying distributional extrapolation to the IRMv1
penalty. Since the original IRMv1 penalty function is inherently non-linear with respect to the data
distribution, it is not directly compatible with distributional extrapolation. To resolve this, the
penalty is reformulated to support extrapolation, ultimately resulting in two novel penalty terms
for IRMv1.

Extensive numerical tests---spanning from small-scale settings based on structural equation models
(SEMs)~\citep{armitage2005encyclopedia} to more practical scenarios involving diverse
computer-vision datasets---demonstrate that the proposed method consistently outperforms
state-of-the-art IRM variants in terms of both accuracy and various calibration metrics. The
implementation of the proposed approach is available at
\url{https://github.com/katoro8989/IRM_Extrapolation}.

To summarize, the key contributions of this work are as follows.
\begin{enumerate}[ label = \textbf{(\roman*)}, itemsep = 0pt, topsep = 0pt ]
\item It is theoretically shown that limited diversity in training environments contributes to the
  suboptimality of the IRMv1 penalty; see \Cref{ch:analysis}.
\item To enhance environment diversity, distributional extrapolation is incorporated into IRMv1,
  leading to two novel penalty terms; see \Cref{ch:method}.
\item The effectiveness of the proposed penalty terms is validated through experiments on synthetic
  SEMs and more realistic, over-parameterized settings using four computer-vision datasets,
  demonstrating superior performance over existing state-of-the-art IRM variants in terms of
  accuracy and calibration metrics; see \Cref{ch:evaluation}.
\item It is further shown that the proposed penalty terms can also serve as effective enhancements
  when integrated into existing IRM variants; see \Cref{ch:evaluation}.
\end{enumerate}

\section{Background}\label{ch:background}

\subsection{Notations}

The model to be trained takes the form of a function $f\colon \mathcal{X} \to \mathcal{Y}$, where
the input space $\mathcal{X} \subset \mathbb{R}^d$ represents the set of input data with dimension
$d$, and the output space $\mathcal{Y} \subset \mathbb{R}$ represents the corresponding set of
ground-truth labels. Data from environment $e \in E$ are assumed to be generated according to a
probability distribution function (PDF) $P_e(x, y)$, where $E$ denotes the set of all possible
environments. For a user-defined loss $\ell \colon \mathcal{Y} \times \mathcal{Y} \to \mathbb{R}$,
for example, $\ell(y^{\prime}, y) \coloneqq (y^{\prime} - y)^2$, $\forall (y^{\prime}, y)\in
\mathcal{Y} \times \mathcal{Y}$, \textit{risk}\/ $R_e(f)$ validates model $f$ according to the
following definition:
\begin{align}\label{eq:risk}
  R_e(f) \coloneqq \mathbb{E}_{(x, y) \sim P_e} \{ \ell(f(x), y) \}\,,
\end{align}
where $\mathbb{E}\{ \cdot\}$ stands for expectation. In the following discussion, it is assumed that
$R_e(f) \geq 0$, $\forall f$.

\subsection{Out-of-distribution generalization}

In machine learning, it is commonly assumed that training and test data are i.i.d. To enhance model
generalization under this assumption, a variety of techniques have been proposed, including weight
decay~\citep{krogh1991simple}, early stopping~\citep{prechelt2002early}, data
augmentation~\citep{perez2017effectivenessdataaugmentationimage}, and ensemble
learning~\citep{Ganaie_2022}. Within the i.i.d.\ framework, empirical risk minimization (ERM) is
widely regarded as one of the most effective approaches for improving generalization. ERM seeks a
model which solves the following task:
\begin{equation}
  \min_{ f\colon \mathcal{X}\to\mathcal{Y} } \sum\nolimits_{ e\in E_{\text{train}} } R_e(f) \,.
  \tag{ERM} \label{eq:erm}
\end{equation}

In the context of deep learning, however, the i.i.d.\ assumption is often violated, and distribution
shifts commonly occur in practice~\citep{quinonero2022dataset}. That is, the training and test data
are drawn from different PDFs. Under such conditions, strategies like ERM---which perform well under
the i.i.d.\ assumption---are known to fail. The task of ensuring model generalization in the
presence of distribution shifts is known as out-of-distribution (OOD) generalization, a challenge
that has gained significant attention in recent years. OOD generalization is typically formulated as
the following $\min$-$\max$ optimization problem~\citep{b2}:
\begin{equation}
  \min_{ f\colon \mathcal{X}\to\mathcal{Y} } \left( R_{\text{OOD}} (f) \coloneqq \max_{e \in E_{
      \text{all}} } R_e(f) \right) \,, \label{eq:ood_loss}
\end{equation}
where $R_{\text{OOD}}$ denotes the worst-case loss across all possible environments, including those
generating the test data. In practice, however, it is rarely feasible to account for all possible
environments during training. As a result, leveraging the limited set of available training
environments effectively to minimize $R_{\text{OOD}}$ becomes a critical objective.

\subsection{Invariant learning}

A central objective in addressing OOD generalization is to ensure that models make consistent
predictions across different environments. This consistency is quantified by the environmental
invariance of model predictions, defined as~\citep{b2}:
\begin{equation}
  \mathbb{E}_{(x, y) \sim P_e} \{y \given f(x)\} = \mathbb{E}_{ (x, y) \sim P_{e^{\prime}} } \{y
  \given f(x)\}\,, \qquad \forall e, e^{\prime} \in E_{ \text{all} } \,, \label{eq:env_invariance}
\end{equation}
that is, the conditional expectation of the label given the model's output remains invariant across
all environments. Intuitively, this implies that the relationship between the model's predictions
and the true labels is preserved regardless of changes in the data distribution.

However, it is important to note that the condition in (\ref{eq:env_invariance}) does not
necessarily guarantee high predictive performance across environments. For instance, a model that
produces random predictions would, by definition, maintain the same random predictive behavior in
every environment and thus satisfy the environmental invariance condition---despite failing to
achieve meaningful accuracy. Furthermore, the well-known lack of environmental invariance in ERM can
be attributed to its tendency to rely not only on invariant features but also on spurious
ones~\citep{b11, sagawa2020distributionallyrobustneuralnetworks}.

\subsection{Invariant risk minimization}

One approach to achieving environmental invariance is through representation learning, where models
are trained to capture only invariant features~\citep{b11, b36}. In this context, let the invariant
feature extractor be defined as $\Phi \colon \mathcal{X} \to \mathcal{H}$, where $\mathcal{H}$ is a
user-defined feature space, and let the predictor (e.g., a classifier) be $\pi \colon \mathcal{H}
\to \mathcal{Y}$. The overall model is then given by the composition $f = \pi \circ \Phi$. Within
this framework, the condition in (\ref{eq:env_invariance}) can be replaced by the following
criterion:
\begin{equation}
  \mathbb{E}_{(x, y) \sim P_e} \{ y \given \Phi(x) \} = \mathbb{E}_{(x, y) \sim P_{e^{\prime}} } \{
  y \given \Phi(x) \}\,, \qquad \forall e, e^{\prime} \in E_{ \text{all} }
  \,. \label{eq:env_invariance_latent}
\end{equation}
The key distinction from (\ref{eq:env_invariance}) is that, in (\ref{eq:env_invariance_latent}),
invariance is enforced with respect to the invariant-feature extractor $\Phi$ rather than the full
model $f$ which depends also on the predictor $\pi$.

Aiming at (\ref{eq:env_invariance_latent}), invariant risk minimization (IRM)~\citep{b11} has been
introduced as the following bi-level optimization task:
\begin{alignat}{2}
  & \min_{ \substack{\Phi\colon \mathcal{X} \to \mathcal{H} \\ \pi\colon \mathcal{H}
      \to \mathcal{Y} } }
  {} && {} \sum\nolimits_{ e\in E_{\text{train}} } R_e (\pi \circ \Phi) \notag \\
  & \text{subject to} \quad
  {} && {} \pi \in \arg\min\nolimits_{ \bar{ \pi} \colon
    \mathcal{H} \to \mathcal{Y} } R_e (\bar{ \pi} \circ \Phi)\,, \forall e\in
  E_{ \text{train} } \,. \label{eq5}\tag{IRM}
\end{alignat}
More specifically, the upper-level problem aims to minimize the sum of losses $R_e$ across all
training environments, similar to ERM. However, unlike ERM, the lower-level problem constrains the
candidate predictor $\pi$ to those that minimize $R_e$ simultaneously across all training
environments. In essence, (\ref{eq5}) targets feature extractors $\Phi$ that facilitate predictors
$\pi$ capable of achieving low loss $R_e$ across all environments, thereby promoting invariant
representation learning through $\Phi$, in line with the criterion
(\ref{eq:env_invariance_latent}). It is worth noting that the lower-level problem is not a
conventional minimization one over an aggregated loss, as is typical in ERM. Instead, it requires
minimizing the loss concurrently across all training environments. This requirement makes the exact
implementation of IRM particularly challenging, motivating the development of various approximation
techniques to make the problem computationally tractable.

\subsection{IRMv1}\label{sec:irmv1}

IRMv1~\citep{b11} is one of the most well-known approximations of IRM. In IRMv1, the feature space
is set as $\mathcal{H} = \mathbb{R}$ in $\Phi \colon \mathcal{X} \to \mathcal{H}$, and the predictor
function $\pi\colon \mathbb{R} \to \mathcal{Y}$ is assumed to be linear. By abuse of notation, this
linear predictor function will henceforth be represented by its slope $\pi \in \mathbb{R}$. In this
context, $f = \pi \circ \Phi$ is reduced to the simple product form $\pi \cdot \Phi$. Since $R_e(
\pi \cdot \Phi ) = R_e(\, (\pi / c) \cdot (c\, \Phi)\, )$ for any nonzero scalar $c$, the problem of
identifying $(\pi, \Phi)$ in (\ref{eq5}) exhibits a scaling ambiguity, that is, if
$(\pi_{\text{opt}}, \Phi_{\text{opt}})$ is a solution of (\ref{eq5}), then $\{\, (\pi_{\text{opt}}
/c, c\, \Phi_{\text{opt}} ) \given c\in \mathbb{R} \setminus\{0\}\, \}$ also solve
(\ref{eq5}). Thus, whenever $\pi_{\text{opt}} \neq 0$, $(1, \pi_{\text{opt}}\cdot
\Phi_{\text{opt}})$ solves (\ref{eq5}). Moreover, when the loss function $R_e(\cdot)$ is convex,
loss $R_e(\pi \cdot \Phi)$ is convex in the scalar $\pi$ for any fixed $\Phi$. As a result, under
the convexity and differentiability of $R_e$, the lower-level optimization in (\ref{eq5}) is
equivalent to ensuring that the gradient with respect to $\pi$ vanishes at the minimizer
$\pi_{\text{opt}}$: $\lvert \nabla_{ \pi \mid \pi = \pi_{ \text{opt} } } R_e(\pi \cdot \Phi)
\rvert^{2} = 0$. Motivated by the aforementioned scaling ambiguity and equivalence, IRMv1 is
formulated as the following single-level approximation of IRM~\citep{b11}:
\begin{equation}
  \min_{\Phi \colon \mathcal{X} \to \mathcal{Y}} \sum\nolimits_{e\in E_{ \text{train} } } \left(\,
  R_e (\Phi) + \lambda\, \lvert \nabla_{ \pi \mid \pi = 1 } R_e( \pi\cdot\Phi) \rvert^2\, \right)
  \,, \tag{IRMv1} \label{eq:irmv1}
\end{equation}
where $\lambda \in \mathbb{R}_{++}$---$\mathbb{R}_{++}$ stands for all positive real numbers---is a
regularization hyperparameter and $R_e(\cdot)$, in general, is not required to be convex but only
differentiable.

IRMv1 has shown promise for out-of-distribution generalization under distribution
shifts~\citep{b11}, and numerous subsequent methods have been developed by building upon its penalty
formulation~\citep{b27, b23, pmlr-v162-zhou22e, b24, b37}. Additional details on these
variants are provided in~\Cref{ch:relatedwork}. Nevertheless, empirical evidence suggests that these
methods frequently fail to outperform a well-tuned ERM baseline~\citep{b29}.

\section{The insufficiency of IRMv1 for invariance guarantees}\label{ch:analysis}

For some $\epsilon \in \mathbb{R}_{++}$, the following relaxation of (\ref{eq:irmv1}) facilitates
the discussion surrounding \Cref{theo:smallEps}, which brings forth the insufficiency of IRMv1 to
identify those feature extractors $\Phi$ that suppress spurious features and promote invariant
ones:
\begin{alignat}{2}
  & \min_{\substack{\Phi \colon \mathcal{X} \to \mathcal{H} \\ \pi \colon
      \mathcal{H} \to \mathcal{Y} } }
  && \sum\nolimits_{ e\in E_{\text{train}} } R_e (\pi \cdot \Phi) \notag\\
  & \text{subject to} \quad &&
   \lvert \nabla_{ \pi } R_e ( \pi \cdot \Phi) \rvert^2 \leq \epsilon\,,
  \forall e \in E_{\text{train}} \,. \label{eq:irmv1_ineq_const}
\end{alignat}

\begin{asu}\label{as:l-smooth}
  In the context of IRMv1 ($\pi$ is considered to be a scalar) and for any $\Phi$, there exists an
  $L_{\Phi} \in \mathbb{R}_{++}$ such that the partial differential operator $\nabla_{ \pi } R_e (
  \pi \cdot \Phi)$ is $L_{\Phi}$-Lipschitz continuous, that is, $\lvert \nabla_{ \pi } R ( \pi_1
  \cdot \Phi) - \nabla_{ \pi } R( \pi_2 \cdot \Phi) \rvert \leq L_{\Phi}\, \lvert \pi_1 - \pi_2
  \rvert$, $\forall \pi_1, \pi_2$.
\end{asu}

\begin{theo}\label{theo:smallEps}
  Presume \Cref{as:l-smooth}. For $\delta \in \mathbb{R}_{++}$, consider the following set of
  parameters $\mathcal{F}_{\delta}$:
  \begin{align*}
    \mathcal{F}_{\delta} \coloneqq \left\{ (\pi, \Phi) \ \bigg| \ \sum\nolimits_{e \in
      E_{\textnormal{train}}} R_e ( \pi \cdot \Phi) \leq \delta \right\} \,.
  \end{align*}
  Choose $\delta$ such that $\mathcal{F}_{\delta} \neq \varnothing$. Then,
  \begin{align*}
    \lvert \nabla_{ \pi } R_e( \pi \cdot \Phi) \rvert^2 \leq 2L_{\Phi} \delta \,,
    \quad \forall (\pi, \Phi) \in \mathcal{F}_{\delta}\,, \forall e \in E_{\textnormal{train}} \,.
  \end{align*}
\end{theo}

\begin{proof}
  See Appendix~\ref{ch:appendix1}.
\end{proof}

Theorem~\ref{theo:smallEps} suggests that when the training environments are highly similar and it
is possible to achieve sufficiently low (though not necessarily zero) ERM loss by relying on
spurious features, the optimality constraint in (\ref{eq:irmv1_ineq_const}) may still be
satisfied. In such cases, the gradient penalty term $\lvert \nabla_{ \pi } R_e (\pi \cdot \Phi)
\rvert^2$ may remain small across the training environments $E_{ \text{train} }$, yet become large
in the test environments $E_{\text{test}}$, indicating overfitting to $E_{\text{train}}$. In other
words, the regularization term in (\ref{eq:irmv1}) or the constraint in (\ref{eq:irmv1_ineq_const})
does not provide strong protection against spurious-feature extractors $\Phi$. Specifically, if such
a $\Phi$, in conjunction with some $\pi$, results in a small ERM loss at the upper level of
(\ref{eq:irmv1_ineq_const}), then according to Theorem~\ref{theo:smallEps}, the lower-level
constraint in (\ref{eq:irmv1_ineq_const}) will also be satisfied for some $\epsilon\ ( \geq
2L_{\Phi} \delta)$. As a result, the spurious-feature extractor $\Phi$ may be incorrectly accepted
as an invariant-feature one.  The claim of Theorem~\ref{theo:smallEps} is justified by recent
studies on over-parameterized settings where it has been shown that even when the condition $\lvert
\nabla_{\pi} R_e (\pi \cdot \Phi) \rvert^2 = 0$ is satisfied, the learned feature extractor $\Phi$
may still depend on spurious features~\citep{b27, pmlr-v162-zhou22e}.

\section{Proposed Method}\label{ch:method}

The preceding discussion reveals a fundamental vulnerability of IRMv1, which becomes pronounced in
practical settings where achieving adequate diversity among training environments is inherently
difficult. This section proposes a method to enhance the diversity of the training data
\textit{without}\/ expanding set $E_{\text{train}}$ through data generation. Motivated by
\Cref{ch:analysis}, two novel loss functions are introduced to address the aforementioned
vulnerability of IRMv1.

\subsection{Risk extrapolation}

\Cref{ch:analysis} suggests that increasing the diversity of $E_{\text{train}}$ is essential for the
practical deployment of IRMv1. When $E_{\text{train}}$ exhibits limited diversity and the generation
of synthetic data is not feasible, it becomes desirable to expand $E_{\text{train}}$ with
pseudo-unseen environments and to perform optimization over this augmented set. Inspired by the work
of \citep{b25}, the present study explores the generation of pseudo-unseen environments through
algorithmic design, without the creation of additional data. As observed in \citep{b25}, risk
$R_e(x, y)$ is linear with respect to the PDF $P_e(x, y)$, implying that an affine combination of
risks across different distributions corresponds to the risk associated with a mixture of those
distributions. Leveraging this insight, \citep{b25} proposed a method of loss extrapolation through
affine combinations---negative coefficients are allowed---of the individual environment losses
$R_e(\cdot)$, thereby representing distributions outside the span of the original training
environments. A subsequent $\min$-$\max$ optimization of the risk over this extrapolated space
enabled training with respect to pseudo-unseen environments.

\subsection{Robustifying invariant learning through penalty extrapolation}

The original IRMv1 regularization loss, $\lvert \nabla_{\pi} R_e (\pi \cdot \Phi) \rvert^2$,
involves computing the gradient of the expected loss and taking its squared norm. Consequently, the
loss is nonlinear with respect to the data distribution $P_e(x, y)$. This non-linearity presents
challenges when attempting to extend the extrapolation approach of \citep{b25}, which relies on
linear combinations of risks, to the IRMv1 setting.

To address this non-linearity issue, the following loss is introduced:
\begin{align}
  \mathcal{J}_{\text{IRM}, e}(\pi, \Phi) \coloneqq \mathbb{E}_{(x, y) \sim P_e} \left\{\, \lvert
    \nabla_{ \pi} \ell( \pi \cdot \Phi(x), y) \rvert^2\, \right\} \,, \quad \forall (\pi, \Phi) \,,
    \forall e\in E_{\textnormal{train}} \,. \label{loss.J}
\end{align}
In contrast to $\lvert \nabla_{\pi} R_e (\pi \cdot \Phi) \rvert^2$---recall that $R_e (\pi \cdot
\Phi) = \mathbb{E}_{(x,y)\sim P_e} \left\{ \ell( \pi \cdot\Phi(x), y) \right \}$---notice that the
squared norm of the gradient is computed prior to taking the expectation under $P_e(x,y)$ in
(\ref{loss.J}). In this way, the loss in (\ref{loss.J}) is linear with respect to the PDF of the
data, which in turn facilitates the implementation of the extrapolation ideas of \citep{b25}. The
following shows that $\mathcal{J}_{\text{IRM}, e}(\cdot, \cdot)$ ``majorizes'' $\lvert \nabla_{\pi}
R_e (\cdot, \cdot) \rvert^2$.

\begin{lem}
  For any risk function $R_e (\pi \cdot \Phi) = \mathbb{E}_{(x,y)\sim P_e} \left\{ \ell( \pi
  \cdot\Phi(x), y) \right \}$, the following holds true:
  \begin{align*}
    \lvert \nabla_{\pi} R_e (\pi \cdot \Phi) \rvert^2 \leq \mathcal{J}_{\textnormal{IRM}, e} (\pi,
    \Phi) \,, \quad \forall (\pi, \Phi)\,, \forall e\in E_{\textnormal{train}} \,.
  \end{align*}
\end{lem}

\begin{proof}
  See Appendix~\ref{sec:proof_lemma41}.
\end{proof}

For notational convenience, and motivated by the discussion in \Cref{sec:irmv1}, let
\begin{align*}
  \mathcal{J}_{\text{IRMv1}, e} (\Phi) \coloneqq \mathcal{J}_{\text{IRM}, e} (1.0, \Phi)\,, \quad
  \forall \Phi\,, \forall e\in E_{\textnormal{train}} \,.
\end{align*}
Based on $\mathcal{J}_{\text{IRMv1}, e}$, this work proposes the following losses
(\ref{eq:mm-irmv1}) and (\ref{eq:v-irmv1}) instead of $\lvert \nabla_{\pi} R_e (\pi \cdot \Phi)
\rvert^2$ in (\ref{eq:irmv1}). First in order is the loss
\begin{alignat}{2}
  \mathcal{C}_{ \text{mm} }(\Phi)
  & {} \coloneqq {} && \max_{ (\alpha_e)_{ e\in E_{\text{train}} } \in \mathcal{A} }
  \sum\nolimits_{\ e \in E_{\text{train}} } \alpha_e\, \mathcal{J}_{\text{IRMv1}, e} (\Phi)
  \notag\\
  & {} = {} && (\, 1 - \alpha_{\min}\, \lvert E_{ \text{train} } \rvert\, )\, \max\nolimits_{e \in
    E_{\text{train}} } \mathcal{J}_{\text{IRMv1}, e} (\Phi) + \alpha_{ \text{min} } \sum\nolimits_{e
    \in E_{ \text{train} } } \mathcal{J}_{\text{IRMv1}, e} (\Phi) \,, \label{eq:mm-irmv1}
\end{alignat}
where $\bm{\alpha} \coloneqq (\alpha_e)_{ e\in E_{\text{train}} }$ stands for a tuple of length
equal to the cardinality $\lvert E_{ \text{train} } \rvert$ of $E_{ \text{train} }$,
\begin{align}
  \mathcal{A} \coloneqq \left \{\, \bm{\alpha} \in \mathbb{R}^{ \lvert E_{
      \text{train} } \rvert }\, \given \sum\nolimits_{ e \in E_{\text{train}} } \alpha_e = 1\,,\,
  \alpha_{\min} \leq \alpha_e\, \right \} \label{constraint.A}
\end{align}
is the intersection of the affine set $\{\, \bm{\alpha} \in \mathbb{R}^{ \lvert E_{ \text{train} }
  \rvert } \given \sum_{ e \in E_{\text{train}} } \alpha_e = 1\, \}$ with the bound constraints
$\{\, \bm{\alpha} \in \mathbb{R}^{ \lvert E_{ \text{train} } \rvert } \given \alpha_{\min} \leq
\alpha_e\, \}$, and the user-defined parameter $\alpha_{\min}$, which \textit{may take negative
  values,} dictates the extent of extrapolation. The proof of the second equality in
(\ref{eq:mm-irmv1}) is provided in Appendix~\ref{ch:appendix_mm_proof}. The intuition behind
(\ref{eq:mm-irmv1}) is to select, from a family of induced extrapolated pseudo-unseen environments,
the distribution that maximizes the penalty. This strategy effectively simulates a broader set of
training environments---without the need for synthetic or augmented data---thereby helping to
mitigate overfitting to spurious features, even when the diversity of training environments is
limited.

Furthermore, \citep{b25} mentions that, in addition to extrapolation, simply adding the variance of
risks across training environments as a regularization term proves to be stable and effective. To
this end, the following alternative loss is also proposed:
\begin{align}
  \mathcal{C}_{\text{v}}(\Phi) \coloneqq \gamma \cdot \text{Var}
  (\, \{\mathcal{J}_{\text{IRMv1}, e} (\Phi) \given e \in E_{\text{train}} \}\, ) +
  \sum\nolimits_{e\in E_{ \text{train} } } \mathcal{J}_{\text{IRMv1}, e} (\Phi)
  \,, \label{eq:v-irmv1}
\end{align}
where $\mathrm{Var}(S)$ represents the empirical variance over a finite set $S$ of real values,
defined by
\begin{align*}
  \mathrm{Var}(S) \coloneqq \frac{1}{ \lvert S\rvert } \sum\nolimits_{s\in S} \left( s - \bar s
  \right)^2 \,, \quad \bar{s} \coloneqq \frac{1}{ \lvert S\rvert } \sum\nolimits_{s\in S} s \,,
\end{align*}
and $\gamma$ is a non-negative scalar that serves as a hyperparameter to determine the extent of
regularization on the variance of the penalty.

To summarize, the following tasks, based on (\ref{eq:mm-irmv1}) and (\ref{eq:v-irmv1}), are proposed
as alternatives to the popular (\ref{eq:irmv1}):
\begin{align}
  & \min_{ \Phi \colon \mathcal{X} \to \mathcal{Y} } \sum\nolimits_{ e \in E_{ \text{train} } } R_e
  (\Phi) + \lambda\, \mathcal{C}_{\text{mm}} (\Phi) \,, \label{eq:mm-irmv1_loss} \tag{mm-IRMv1} \\
  & \min_{ \Phi \colon \mathcal{X} \to \mathcal{Y} } \sum\nolimits_{ e \in E_{ \text{train} } } R_e
  (\Phi) + \lambda\, \mathcal{C}_\text{v} (\Phi) \,, \label{eq:v-irmv1_loss} \tag{v-IRMv1}
\end{align}
where $\lambda\in \mathbb{R}_{++}$ is a user-defined regularization parameter.

\section{Numerical Tests}\label{ch:evaluation}

\subsection{Structural equation models}\label{subsec:non-over}

Here, we conduct experiments using structural equation models (SEMs) in a scenario where the
training environments are highly similar to each other, so that spurious features can misleadingly
reduce the training loss.

\subsubsection{Setting}

The following SEM, taken from \citep{b11}, is considered:
\begin{alignat}{3}
  \mathbf{x}^{ \text{inv} }_e & \sim && \mathcal{N}( \mathbf{0}_d, e^2 \mathbf{I}_d ) \,, \notag \\
  y_e & {} \coloneqq {} && \mathbf{1}_d^{\top} \mathbf{x}^{ \text{inv} }_e + u \,, \quad && u \sim
  \mathcal{N}(0, I_1) \,, \notag \\
  \mathbf{x}^{ \text{spu} }_e & \coloneqq && y_e \cdot \mathbf{1}_d + \mathbf{v}_e \,,
  \quad && \mathbf{v}_e \sim \mathcal{N}( \mathbf{0}_d, e^2 \mathbf{I}_d) \,. \label{sem.spu}
\end{alignat}
where $\mathbf{x}^{ \text{inv} }_e, \mathbf{x}^{ \text{spu} }_e$ and $y_e$ are $d\times 1$
vector-valued and scalar-valued realizations of random variables, respectively, $\mathcal{N}(
\mathbf{0}, e^2 \mathbf{I}_d )$ stands for the $d$-dimensional normal PDF with mean the $d\times 1$
all-zero vector $\mathbf{0}_d$ and covariance matrix $e^2 \mathbf{I}_d$, $\mathbf{I}_d$ is the
$d\times d$ identity matrix, $\mathbf{1}_d$ is the $d\times 1$ all-one vector, and $\top$ denotes
vector/matrix transposition. Each environment is uniquely characterized by a distinct real value
$e$.

The following estimation task is considered: given the $2d \times 1$ vector $\mathbf{x}_e \coloneqq
[ \mathbf{x}_e^{\text{inv} \top}, \mathbf{x}_e^{\text{spu} \top}]^{\top}$, estimate $y_e$ by
$\hat{y}_e \coloneqq \hat{\mathbf{w}}_{\text{inv}}^{\top } \mathbf{x}_e^{\text{inv}} +
\hat{\mathbf{w}}_{\text{spu}}^{\top } \mathbf{x}_e^{\text{spu}}$, where the $d\times 1$ parameter
vectors $\hat{\mathbf{w}}_{\text{inv}}, \hat{\mathbf{w}}_{\text{spu}}$ of the estimation model need
to be identified. In the case where $(\hat{\mathbf{w}}_{\text{inv}}, \hat{\mathbf{w}}_{\text{spu}})
= ( \mathbf{1}_d, \mathbf{0}_d)$, that is, estimation is based solely on the invariant feature
$\mathbf{x}_e^{\text{inv}}$, then the estimation error $y_e - \hat{y}_e = y_e - \mathbf{1}_d^{\top}
\mathbf{x}_e^{\text{inv}} = u \sim \mathcal{N}(0, 1)$ becomes environment invariant, so that
$\mathbb{E}\{ y_e - \hat{y}_e \} = 0$ and $\mathbb{E}\{ ( y_e - \hat{y}_e )^2 \} = 1$, regardless
of the value of $e$. On the other hand, any attempt to utilize the spurious $\mathbf{x}^{ \text{spu}
}_e$ in the estimation model by using a non-zero $\hat{\mathbf{w}}_{\text{spu}}$ renders the
estimation error environment dependent due to (\ref{sem.spu}).

Tests are conducted for $d = 5$. To measure invariance, the causal error $(1/d) \norm{
  \hat{\mathbf{w}}_{\text{inv}} - \mathbf{1}_d }^2$ and the non-causal one $(1/d) \norm{
  \hat{\mathbf{w}}_{\text{spu}} - \mathbf{0}_d }^2$ are employed \citep{b11}. Set $E_{\text{train}}$
comprises two environments, and three settings are investigated for $E_{\text{train}}$: $\{0.2, 2\},
\{0.2, 1\}, \{0.2, 0.6\}$. As the two values of $e$ are getting closer to each other,
$\mathbf{x}^{\text{spu}}_e$ becomes spuriously invariant within $E_{\text{train}}$, making it
increasingly difficult to identify $\mathbf{x}^{\text{inv}}_e$. 
The model was trained using the mean squared error (MSE) loss function, and hyperparameter tuning was performed based on validation data within the training environment. Further details are provided in Section~\ref{subsec:sem_settings}.

\subsubsection{Results}

Results are shown in \Cref{tab:sem_result,tab:sem_result_others}. A common trend observed across all
methods is that as the $e$ values in the training environments become more similar, the error rates
increase, making it more challenging to achieve invariance. However, while IRMv1 experiences
particularly severe performance degradation, the proposed methods exhibit consistent improvements
across all settings. Notably, mm-IRMv1 achieves the most substantial gains, with reductions of up to
approximately 72\% in causal error and 56\% in non-causal error. Results under settings with a
larger number of training environments but limited diversity are presented in
\Cref{tab:sem_result_others}. Even under these conditions, the proposed methods consistently
outperform IRMv1. These findings confirm---from the perspective of learned parameters---that
extrapolating the IRMv1 penalty fosters invariant learning, even when training environment diversity
is limited.

\begin{table}[ht]
  \caption{Invariance errors in SEMs. Even in scenarios where the training environments are
    similar---making it difficult to eliminate spurious features---the proposed methods,
    particularly mm-IRMv1, consistently achieve substantial improvements over the IRMv1
    baseline. Percentages indicate performance changes relative to IRMv1.}

  \renewcommand{\arraystretch}{1.2}
  \resizebox{\textwidth}{!}{%
    \begin{tabular}{l|ll|ll|ll}
      \toprule
      & \multicolumn{2}{c|}{$E_{\text{train}}=\{0.2,2\}$}
      & \multicolumn{2}{c|}{$E_{\text{train}}=\{0.2,1\}$}
      & \multicolumn{2}{c}{$E_{\text{train}}=\{0.2,0.6\}$}\\
      & causal err ($\downarrow$) & non‑causal err ($\downarrow$)
      & causal err ($\downarrow$) & non‑causal err ($\downarrow$)
      & causal err ($\downarrow$) & non‑causal err ($\downarrow$)\\
      \midrule
      IRMv1
      & 0.487 $\pm$ 0.840 & 0.205 $\pm$ 0.351
      & 0.798 $\pm$ 0.152 & 0.464 $\pm$ 0.026
      & 1.418 $\pm$ 0.091 & 0.686 $\pm$ 0.068\\
      \rowcolor{gray!15}
      v-IRMv1 (Ours)
      & \textbf{0.414 $\pm$ 0.650} {\color{green}(-15.0\%)}   
      & 0.218 $\pm$ 0.330 {\color{red}(+6.3\%)}            
      & \textbf{0.503 $\pm$ 0.137} {\color{green}(-36.9\%)} 
      & \textbf{0.360 $\pm$ 0.082} {\color{green}(-22.4\%)} 
      & \textbf{1.151 $\pm$ 0.019} {\color{green}(-18.8\%)} 
      & \textbf{0.581 $\pm$ 0.027} {\color{green}(-15.3\%)} \\
      \rowcolor{gray!15}
      mm-IRMv1 (Ours)
      & \textbf{0.218 $\pm$ 0.373} {\color{green}(-55.2\%)} 
      & \textbf{0.131 $\pm$ 0.224} {\color{green}(-36.1\%)} 
      & \textbf{0.222 $\pm$ 0.059} {\color{green}(-72.2\%)} 
      & \textbf{0.206 $\pm$ 0.059} {\color{green}(-55.6\%)} 
      & \textbf{1.006 $\pm$ 0.311} {\color{green}(-29.1\%)} 
      & \textbf{0.564 $\pm$ 0.153} {\color{green}(-17.8\%)} \\
      \bottomrule
    \end{tabular}
    \label{tab:sem_result}
  }
\end{table}

\begin{table}[t]
  \centering
  \renewcommand{\arraystretch}{0.9}
  \caption{Test accuracy (\%) is reported for the four datasets individually, along with their
    average. For each method, ``base'' refers to the original IRM variant, while ``mm'' and ``v''
    denote the versions that incorporate the penalties defined in (\ref{eq:mm-irmv1}) and
    (\ref{eq:v-irmv1}), respectively. On average, the proposed approach enhances performance across
    all IRM variants, with the ``v'' variant consistently outperforming the original. Percentages
    indicate performance changes relative to the baseline method.}
  \label{tab:acc}
  \resizebox{1.0\textwidth}{!}{
    \tiny
    \begin{tabular}{lccccc|l}
      \toprule
      Method & {} & {CMNIST} & {CFMNIST} & {PACS} & {VLCS} & \textbf{Avg.}\\
      \midrule
      ERM &  & 30.9$\pm$0.6 & 28.4$\pm$0.1 & 76.7$\pm$0.6 & 57.7$\pm$0.3 & 48.4\\
      \midrule
      & base & 64.7$\pm$0.5 & 74.3$\pm$1.1 & 75.5$\pm$1.4 & 58.4$\pm$0.5 & 68.2\\
      \rowcolor{gray!15} \multicolumn{1}{>{\cellcolor{white}}l}{IRMv1}  & v & {68.1$\pm$0.4} &
               {74.8$\pm$0.5} & {75.9$\pm$3.8} & 58.4$\pm$2.1 & \textbf{69.3}
               {\color{green}(+1.6\%)} \\
      \rowcolor{gray!15} \multicolumn{1}{>{\cellcolor{white}}l}{}  & mm & {66.8}$\pm$0.6 &
      73.5$\pm$0.4 & 72.7$\pm$3.9 & {59.0$\pm$1.1} & 68.0 {\color{red}(-0.3\%)}\\
      \midrule
      & base & 66.5$\pm$0.5 & 75.5$\pm$0.9 & 76.3$\pm$2.2 & 56.7$\pm$1.9 & 68.8\\
      \rowcolor{gray!15} \multicolumn{1}{>{\cellcolor{white}}l}{BIRM}  & v & {69.0$\pm$0.5} &
               {75.9$\pm$0.7} & {76.5$\pm$3.1} & 56.6$\pm$1.8 & \textbf{69.5}
               {\color{green}(+1.0\%)}\\
      \rowcolor{gray!15} \multicolumn{1}{>{\cellcolor{white}}l}{}  & mm & 66.2$\pm$0.6 &
      73.4$\pm$0.2 & 73.4$\pm$0.2 & 56.7$\pm$0.6 & 67.6 {\color{red}(-1.7\%)}\\
      \midrule
      & base & 67.2$\pm$0.4 & 67.3$\pm$2.2 & 71.3$\pm$1.9 & 51.2$\pm$5.4 & 64.2\\
      \rowcolor{gray!15} \multicolumn{1}{>{\cellcolor{white}}l}{BLO}  & v & 67.0$\pm$1.9 &
               {70.5$\pm$2.1} & {71.4$\pm$4.6} & {53.8$\pm$5.4}& \textbf{65.7}
               {\color{green}(+2.3\%)}\\
      \rowcolor{gray!15} \multicolumn{1}{>{\cellcolor{white}}l}{}  & mm & {67.9$\pm$0.5} &
               {71.2$\pm$1.8} & 69.4$\pm$1.2 & {53.4$\pm$6.0} & \textbf{65.5}
               {\color{green}(+2.0\%)}\\
      \bottomrule
    \end{tabular}
  }
\end{table}

\subsection{Vision Datasets}

The proposed framework is evaluated in the more typical and over-parameterized deep-learning
scenarios for computer vision.

\subsubsection{Setting}

The proposed methods are evaluated on four vision classification datasets: Colored MNIST (CMNIST)
\citep{b2}, Colored FashionMNIST (CFMNIST) \citep{b26}, PACS \citep{b31}, and VLCS
\citep{torralba2011unbiased}. Out-of-distribution (OOD) datasets generally differ in the level of
generalization difficulty, depending on the nature of the distributional shift. Specifically, CMNIST
and CFMNIST exhibit correlation shifts, while PACS and VLCS are characterized by diversity shifts
\citep{ye2022ood}. A ResNet-18 model \citep{he2016deep} is employed as the feature
extractor $\Phi$, whose parameters need to be identified, across all datasets.

The evaluation metrics include accuracy, expected calibration error (ECE), and adapted ECE (ACE) \citep{nixon2019measuring}, all computed on
the test environments. While accuracy measures predictive performance, ECE and ACE specifically
assess the calibration of the model's confidence---that is, how well the predicted confidence aligns
with actual correctness. These calibration metrics provide a more nuanced understanding of model
performance. Recent studies have shown that confidence calibration across multiple environments
guarantees the optimality of IRM \citep{wald2021on}, and their empirical connections have been
demonstrated by several works \citep{ovadia2019can, immer2021scalable, yoshida2024towards, naganuma2023empirical}.

The competing methods include empirical risk minimization (ERM), IRMv1, and two recent
state-of-the-art IRM variants: Bayesian IRM (BIRM) \citep{b27} and BLOC-IRM (BLO) \citep{b37}. Since
BIRM and BLO are based on the losses $\mathcal{J}_{\text{IRMv1}, e}(\Phi)$ and
$\mathcal{J}_{\text{IRM}, e}(\Phi)$, respectively, they are compatible with the proposed
extrapolation method. Accordingly, in addition to IRMv1, the effectiveness of the extrapolation
approach is also evaluated when integrated with BIRM and BLO.

Additional details can be found in \Cref{ch:appendix2}.

\begin{table}[t]
  \centering
  \renewcommand{\arraystretch}{0.9}
  \caption{Test ECE is reported for the four datasets and their average. For each method, ``base''
    denotes the original IRM variant, while ``mm'' and ``v'' refer to the versions incorporating the
    proposed penalties defined in (\ref{eq:mm-irmv1}) and (\ref{eq:v-irmv1}), respectively. On
    average, the proposed approach improves calibration performance across all IRM
    variants. Percentages indicate performance changes relative to the baseline method.}
  \label{tab:ece}
  \resizebox{1.0\textwidth}{!}{
    \tiny
    \begin{tabular}{llcccc|l}
      \toprule
      Method & {} & {CMNIST} & {CFMNIST} & {PACS} & {VLCS} & \textbf{Avg.}\\
      \midrule
      ERM &  & 57.9$\pm$1.1 & 54.4$\pm$0.3 & 12.8$\pm$1.0 & 22.1$\pm$0.1 & 36.8\\
      \midrule
      & base & 10.4$\pm$0.3 & 17.8$\pm$0.6 & 14.7$\pm$1.5 & 22.8$\pm$0.2 & 16.4\\
      \rowcolor{gray!15} \multicolumn{1}{>{\cellcolor{white}}l}{IRMv1}  & v & 10.5$\pm$0.5 &
      17.9$\pm$1.0 & {12.9$\pm$2.8} & 23.0$\pm$1.3 & \textbf{16.1} {\color{green}(-1.8\%)}\\
      \rowcolor{gray!15} \multicolumn{1}{>{\cellcolor{white}}l}{}  & mm & {10.5}$\pm$0.9 &
      13.8$\pm$0.6 & 16.6$\pm$3.6 & {19.8$\pm$0.6} & \textbf{15.2} {\color{green}(-7.3\%)}\\
      \midrule
      & base & 10.1$\pm$0.9 & 19.6$\pm$0.7 & {12.9$\pm$1.3} & 22.8$\pm$1.0 & 16.3\\
      \rowcolor{gray!15} \multicolumn{1}{>{\cellcolor{white}}l}{BIRM}  & v & {10.0$\pm$0.4} &
               {19.4$\pm$0.7} & 13.5$\pm$2.8 & 23.5$\pm$0.2 & 16.6 {\color{red}(+1.8\%)}\\
      \rowcolor{gray!15} \multicolumn{1}{>{\cellcolor{white}}l}{}  & mm & 9.5$\pm$0.4 &
               {13.4$\pm$0.6} & 13.4$\pm$0.6 & {22.9$\pm$1.5} & \textbf{15.3}
               {\color{green}(-6.1\%)}\\
      \midrule
      & base & 13.7$\pm$0.2 & 10.4$\pm$0.7 & 18.4$\pm$2.7 & 17.8$\pm$3.6 & 15.1\\
      \rowcolor{gray!15} \multicolumn{1}{>{\cellcolor{white}}l}{BLO}  & v & {11.6$\pm$1.7} &
               {11.6$\pm$1.7} & 18.6$\pm$6.0 & {15.6$\pm$1.9}& \textbf{14.7}
               {\color{green}(-2.6\%)}\\
      \rowcolor{gray!15} \multicolumn{1}{>{\cellcolor{white}}l}{}  & mm & {13.0$\pm$0.2} &
               {13.7$\pm$2.3} & 19.0$\pm$2.3 & {14.1$\pm$2.9} & \textbf{14.9}
               {\color{green}(-1.3\%)} \\
      \bottomrule
    \end{tabular}
  }
\end{table}

\begin{table}[t]
  \centering
  \renewcommand{\arraystretch}{0.9}
  \caption{Test adapted calibration error (ACE) is reported for the four datasets, along with their
    average. For each method, ``base'' denotes the original IRM variant, while ``mm'' and ``v''
    represent the variants incorporating the proposed penalties defined in (\ref{eq:mm-irmv1}) and
    (\ref{eq:v-irmv1}), respectively. On average, the proposed approach improves calibration
    performance across all IRM variants. Percentages indicate performance changes relative to the
    baseline method.}
  \label{tab:ace}
  \resizebox{1.0\textwidth}{!}{
    \tiny
    \begin{tabular}{llcccc|l}
      \toprule
      Method & {} & {CMNIST} & {CFMNIST} & {PACS} & {VLCS} & \textbf{Avg.}\\
      \midrule
      ERM &  & 57.8$\pm$1.1 & 54.5$\pm$0.3 & 12.3$\pm$0.7 & 22.0$\pm$0.1 & 36.6\\
      \midrule
      & base & 14.0$\pm$0.3 & 20.2$\pm$0.1 & 14.1$\pm$1.7 & 22.6$\pm$0.1 & 17.7\\
      \rowcolor{gray!15} \multicolumn{1}{>{\cellcolor{white}}l}{IRMv1}  & v & {14.0$\pm$0.4} &
      20.2$\pm$0.3 & {12.4$\pm$2.6} & 22.6$\pm$1.4 & \textbf{17.3} {\color{green}(-2.3\%)}\\
      \rowcolor{gray!15} \multicolumn{1}{>{\cellcolor{white}}l}{}  & mm & {13.8}$\pm$1.0 &
               {17.7$\pm$0.5} & 16.1$\pm$3.7 & {19.9$\pm$0.7} &
               \textbf{16.9}{\color{green}(-4.5\%)}\\
      \midrule
      & base & 13.8$\pm$1.0 & 21.0$\pm$0.4 & {12.4$\pm$1.1} & 22.9$\pm$0.9 & 17.5\\
      \rowcolor{gray!15} \multicolumn{1}{>{\cellcolor{white}}l}{BIRM}  & v & {13.3$\pm$0.5} &
               {20.7$\pm$0.5} & 13.2$\pm$2.8 & 23.1$\pm$0.4 & 17.6 {\color{red}(+0.6\%)}\\
      \rowcolor{gray!15} \multicolumn{1}{>{\cellcolor{white}}l}{}  & mm & 13.0$\pm$0.0 &
               {16.7$\pm$0.5} & 12.4$\pm$1.1 & {23.0$\pm$1.6} & \textbf{16.8}
               {\color{green}(-4.0\%)}\\
      \midrule
      & base & 16.5$\pm$0.2 & 14.9$\pm$0.3 & 17.9$\pm$2.7 & 19.1$\pm$3.5 & 17.1\\
      \rowcolor{gray!15} \multicolumn{1}{>{\cellcolor{white}}l}{BLO}  & v & {14.6$\pm$1.3} &
               {14.6$\pm$1.3} & 18.0$\pm$6.2 & {16.7$\pm$1.7}& \textbf{16.5}
               {\color{green}(-3.5\%)}\\
      \rowcolor{gray!15} \multicolumn{1}{>{\cellcolor{white}}l}{}  & mm & {15.6$\pm$0.2} &
               {17.2$\pm$2.2} & 18.4$\pm$2.3 & {15.6$\pm$2.6} & \textbf{16.7}
               {\color{green}(-2.3\%)}\\
      \bottomrule
    \end{tabular}
  }
\end{table}

\subsubsection{Results}

\textbf{Accuracy.} \Cref{tab:acc} presents the test accuracy of each method across the four
datasets, along with their averages. On average, at least one of the proposed variants consistently
outperforms the original method for every IRM variant. Notably, the ``v'' variant consistently
surpasses the original methods. Compared to the best-performing original baseline (the base variant
of BIRM at 68.8\%), the proposed approach achieves an improvement of 1.0\% with
the v-BIRM variant.

\textbf{Calibration Metrics.} Next, we show the results for calibration metrics in
\Cref{tab:ece,tab:ace}. Specifically, \Cref{tab:ece} reports results for ECE, and \Cref{tab:ace} for
ACE. Across both metrics, our extrapolation methods consistently improve upon each original
variant. The improvement in calibration metrics suggests that our distributional extrapolation
approach effectively mitigates the overconfidence observed in recent neural
networks~\citep{minderer2021revisiting}. However, unlike the comparison based on accuracy, we did
not observe a consistent trend regarding the superiority of ``v'' or ``mm''. Notably, when focusing
solely on IRMv1, the variant ``v'' consistently outperformed the original one, aligning with the
trend observed in accuracy.

\textbf{Extrapolation Mitigates IRMv1 Overfitting.} Furthermore, the study examines whether the
IRMv1 penalty term leads to overfitting to the training environment, as discussed in
\Cref{ch:analysis}, and whether the proposed methods effectively mitigate this issue using the
CMNIST dataset. \Cref{fig:vertical_images} presents scatter plots showing the relationships between
the IRMv1 penalty values in the training environment and various evaluation metrics. In these plots,
blue circles represent the original IRMv1, green squares denote v-IRMv1, and red triangles
correspond to mm-IRMv1. Each point reflects the recorded values at each epoch during the final 50
epochs for each method.

Notably, the penalty values for the original IRMv1 approached zero over the final 50 epochs;
however, its performance on test environment metrics was inferior to that of the proposed
methods. This finding supports the vulnerability of IRMv1 discussed in \Cref{ch:analysis}. In
contrast, while the penalty values for the proposed methods did not decline as sharply as those of
IRMv1, they consistently achieved superior performance across all test metrics, indicating improved
out-of-distribution generalization. These results demonstrate that the proposed distribution
extrapolation approach effectively prevents overfitting to the penalty in the training environment,
even in scenarios with limited environmental diversity such as CMNIST, thereby promoting the
learning of truly invariant features.

Additional visualizations and scatter plots on other IRM variants can be found in
\Cref{sec:additional_penalty_relation}

\begin{figure}[t]
  \vspace{-3mm}
  \centering
  \begin{subfigure}[b]{0.33\linewidth}
    \includegraphics[width=\linewidth]{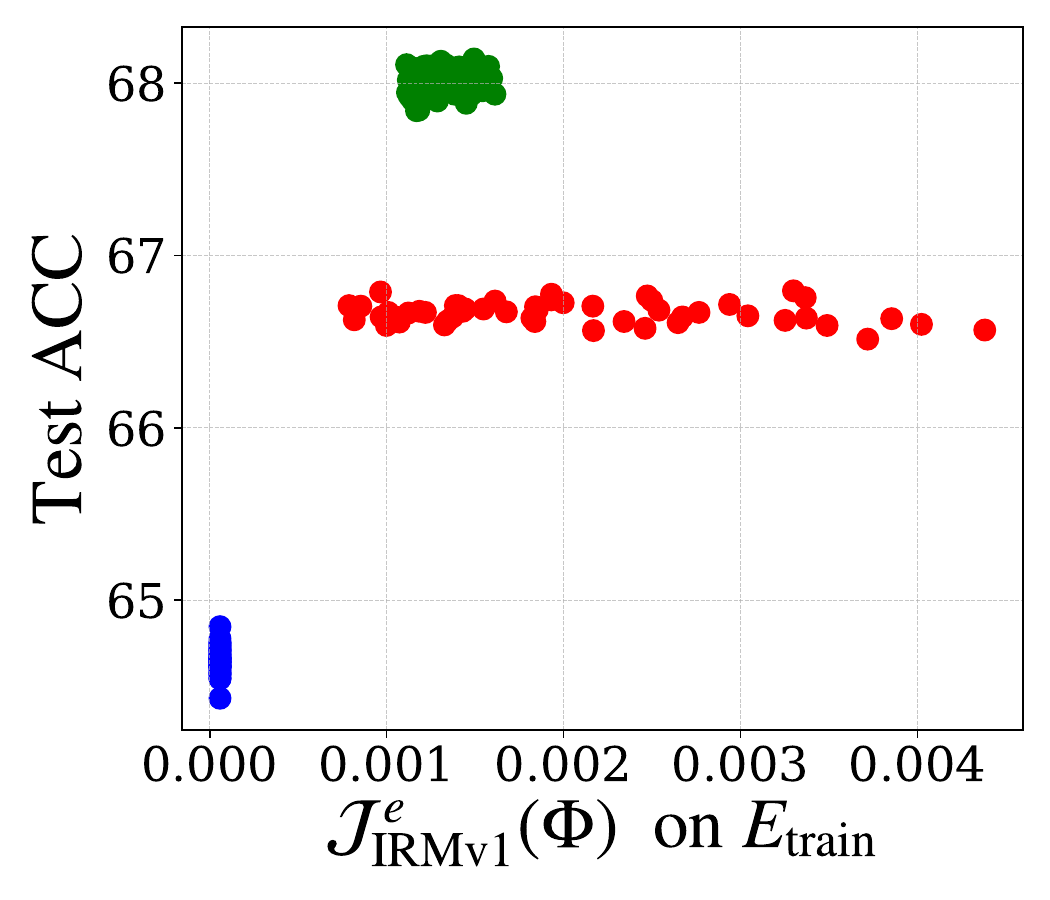}
    \label{fig:cmnist_extend1}
  \end{subfigure}%
  \begin{subfigure}[b]{0.33\linewidth}
    \includegraphics[width=\linewidth]{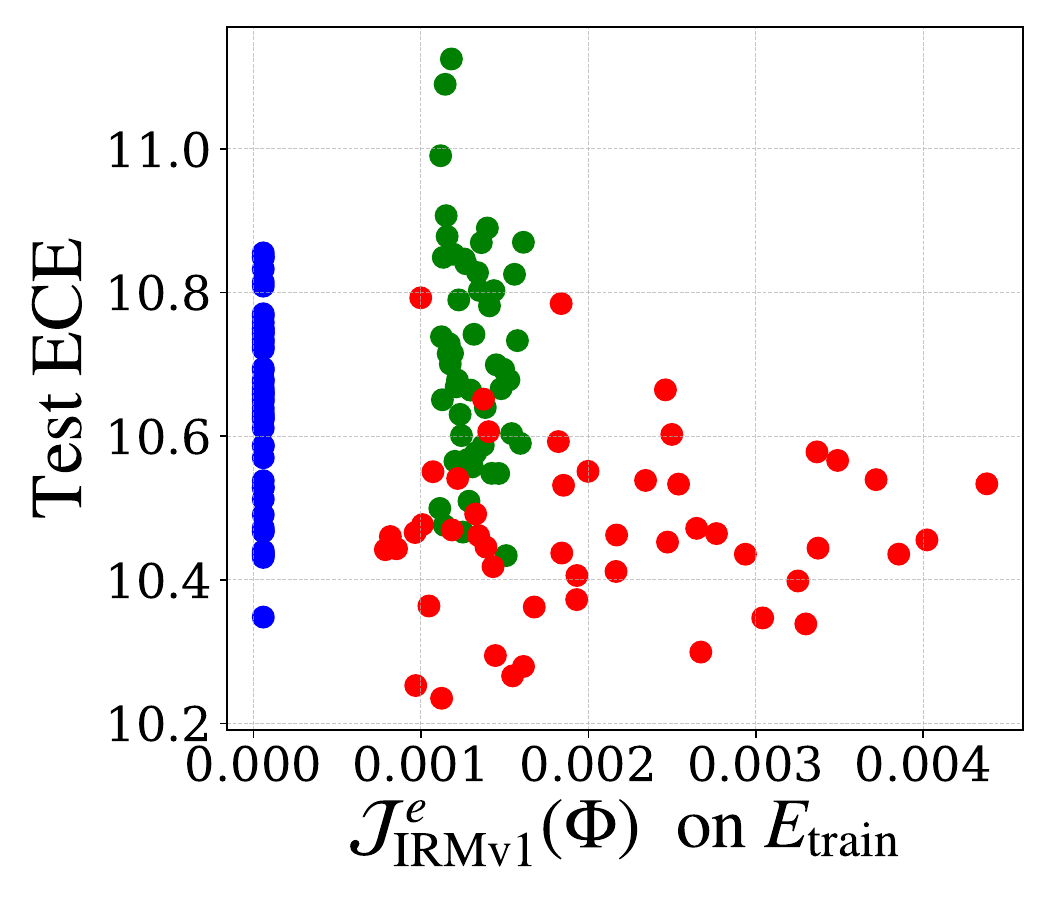}
    \label{fig:cobj_grad1}
  \end{subfigure}%
  \begin{subfigure}[b]{0.33\linewidth}
    \includegraphics[width=\linewidth]{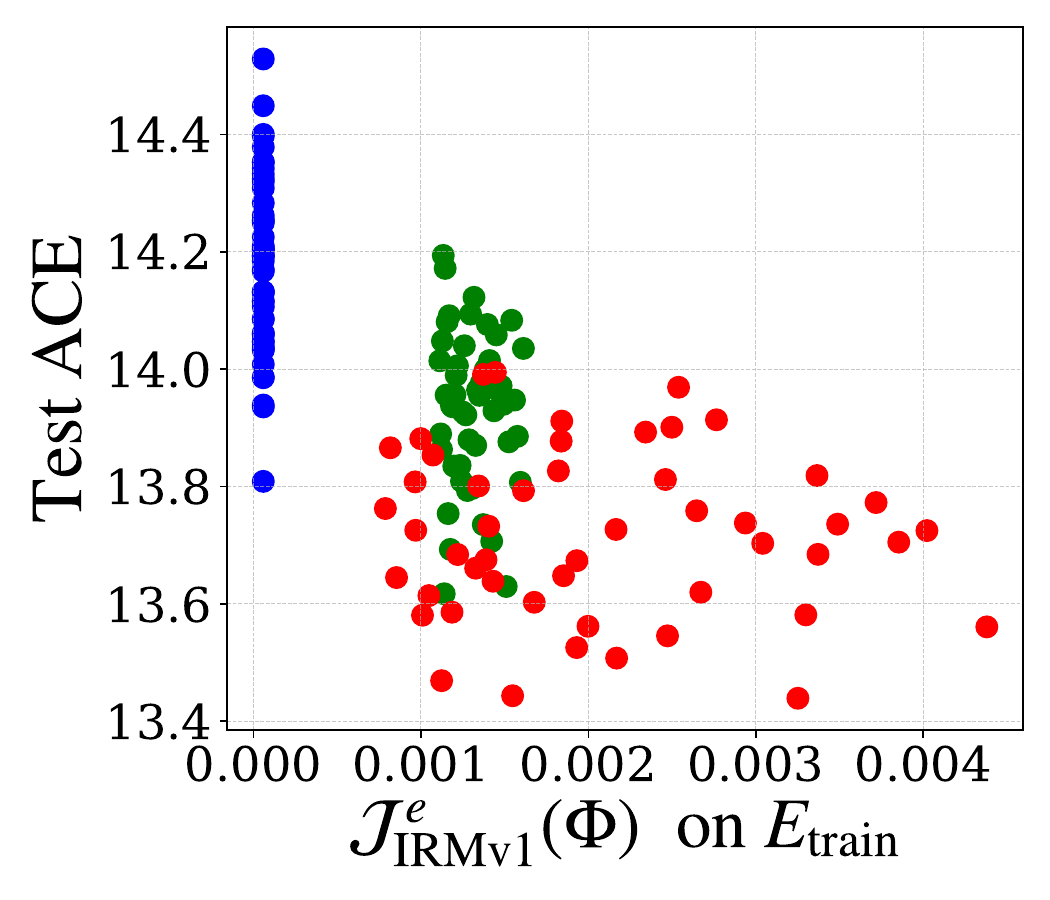}
    \label{fig:cobj_grad2}
  \end{subfigure}%
  \vspace{-5mm}
  \begin{center}
    \begin{tabular}{clclcl}
      \tikz \draw[blue,fill=blue] (0,0) circle (0.5ex); & IRMv1&
      \tikz \draw[matplotlibgreen,fill=matplotlibgreen] (0,0) circle (0.5ex); & v-IRMv1 (Ours)&
      \tikz \draw[red,fill=red](0,0)  circle (0.5ex); & mm-IRMv1 (Ours)\\
    \end{tabular}
  \end{center}

  \caption{The figure illustrates the relationship between the IRMv1 penalty values in the training
    environment and the corresponding test evaluation metrics---accuracy, ECE, and ACE (from left to
    right)---on the CMNIST dataset. Each point represents the recorded values at each epoch during
    the final 50 epochs for each method. Although IRMv1 effectively reduces the training penalty to
    near zero, its performance across all test metrics remains suboptimal, reinforcing the
    vulnerability highlighted in \Cref{ch:analysis}. In contrast, the proposed distributional
    extrapolation methods mitigate overfitting of the IRMv1 penalty to the training environment and
    consistently yield improved performance across all evaluation metrics.}
  \label{fig:vertical_images}
  \vspace{-2mm}
\end{figure}

\section{Related Work}\label{ch:relatedwork}

\subsection{Learning Methods for OOD Generalization}

In recent practical applications of deep neural networks, violations of the i.i.d.\ (independent and
identically distributed) assumption have become increasingly common, making the improvement of
generalization performance under out-of-distribution (OOD) settings a critical challenge. To address
this issue, learning algorithm-based approaches have emerged as one of the mainstream solutions.

Representation learning–based approaches are a widely studied area in the context of OOD
generalization, aiming to extract invariant features from raw data that are independent of
environmental variations. Prominent examples include IRM \citep{b11}, which relies on bi-level
optimization, and domain-adversarial neural networks (DANN)
\citep{b36}, which employ adversarial learning
techniques.

In addition to representation learning, calibration-based approaches have also garnered increasing
attention. Calibration refers to the alignment between a model’s confidence and its predictive
accuracy, and has been theoretically linked to OOD generalization \citep{wald2021on,
  yoshida2024towards}. \citep{wald2021on} demonstrated a theoretical connection between calibration
and model invariance, showing that incorporating a calibration-focused regularization term into ERM
can enhance OOD generalization performance.

Moreover, methods based on min-max optimization of the loss function have been proposed for OOD
generalization. \citep{sagawa2020distributionallyrobustneuralnetworks} applied the concept of
distributionally robust optimization (DRO) \citep{ben2013robust}---which minimizes loss under the
worst-case distribution within the training data---to deep learning and introduced group
DRO. However, \citep{b25} highlighted the overfitting tendency of traditional DRO approaches, which
are confined to training environments, and proposed a method that performs DRO over a broader
distribution set, closer to potential test environments, via loss extrapolation.

\subsection{Variants of IRM}

IRM \citep{b11}, introduced as a representation learning framework for OOD scenarios, formulates the
objective as a bi-level optimization problem, which is known to be computationally challenging. To
address this difficulty, a variety of approximation methods have been developed.

The most fundamental of these is IRMv1, which assumes a linear form for the predictor $\pi$ and
reformulates the original problem as a single-level optimization via a penalization
approach. However, IRMv1 has been criticized for its tendency to overfit to training
environments. In response, Bayesian IRM (BIRM) \citep{b27} adopts Bayesian inference to encourage
invariance in the posterior distributions, aiming to mitigate overfitting. Information
bottleneck-based IRM (IB-IRM) \citep{b23} improves the learning of invariant representations by
applying the information bottleneck principle directly to the feature extractor $\Phi$
\citep{yoshida2024towards}.

Alternative formulations based on game theory have also been proposed. IRM game \citep{b26}
formulates the learning objective as a Nash equilibrium over environment-specific losses, while
Pareto IRM (PAIR) \citep{b24} seeks Pareto-optimal trade-offs among the losses of ERM, IRMv1, and
REx to balance in-distribution and OOD performance.

To address the scalability challenges posed by over-parameterized neural networks,
\citep{pmlr-v162-zhou22e} demonstrated that IRMv1 can be effectively applied to large models by
suppressing their fitting capacity through pruning, thereby preserving invariance. In addition,
\citep{b37} proposed a novel bi-level optimization strategy that simplifies the
lower-level problem of IRM, diverging from conventional single-level approximation methods.

Despite their theoretical motivations, existing approximation methods have been reported to
underperform compared to well-tuned ERM in practice \citep{b29}, underscoring the need for more
robust and principled IRM approximations. To this end, the present study introduces improved penalty
formulations for IRMv1, aiming to enhance the performance of existing approximation methods that
rely on the original IRMv1 penalty.

\section{Discussion}\label{ch:limitation}

\subsection{v-IRMv1 or mm-IRMv1?}

Although mm-IRMv1 consistently outperformed v-IRMv1 in the SEM-based experiments, the opposite trend
was observed on realistic vision datasets, where v-IRMv1 demonstrated more stable and superior
performance. This discrepancy is attributed to the numerical instability inherent in closed-form
min-max optimization problems, as noted by \citep{b25}.

Specifically, the bottleneck arises from the $\max$ operator in (\ref{eq:mm-irmv1}): although each
individual loss $\mathcal{J}_{\text{IRMv1}, e}(\Phi)$ may be smooth in $\Phi$ for all $e$, function
$\max\nolimits_{e \in E_{\text{train}}} \mathcal{J}_{\text{IRMv1}, e}(\Phi)$ is not necessarily
smooth and may exhibit a complex, non-differentiable landscape (graph) that is difficult to
optimize. During training on realistic datasets, the environments $\arg \max\nolimits_{e \in
  E_{\text{train}}} \mathcal{J}_{\text{IRMv1}, e}(\Phi)$ that achieve the maximum value can change
frequently and abruptly. This results into abrupt gradient shifts, numerical instability and
increased difficulty in optimization.

The pronounced success of mm-IRMv1 in the SEM-based experiments is hypothesized to stem from the
model's low dimensionality. While the vision-dataset experiments employed ResNet-18 in an
overparameterized regime, the SEM setting used a model with only ten learnable parameters---five for
$\hat{\mathbf{w}}_{\text{inv}}$ and five for $\hat{\mathbf{w}}_{\text{spu}}$---representing a
substantially lower-dimensional hypothesis space. This significant reduction in degrees of freedom
likely renders the model less sensitive to the non-smoothness introduced by the mm-IRMv1 objective,
resulting in a less complex optimization landscape and more stable convergence.

As an ablation study on CMNIST, the proposed methods were evaluated using both a 3-layer MLP and
ResNet-18. Table~\ref{tab:wrap_dimensionality} illustrates how model dimensionality influences
the test accuracy of mm-IRMv1. In the low-dimensional setting with the 3-layer MLP, mm-IRMv1 clearly
outperforms v-IRMv1. However, in the higher-capacity, overparameterized ResNet-18, the effects of
non-smoothness become more pronounced, and v-IRMv1 tends to yield superior performance.

\begin{table}
  \vspace{-5mm}
  \centering
  \caption{Test accuracy of mm-IRMv1 as a function of model dimensionality on CMNIST. In the
    low-dimensional and less flexible setting with a 3-layer MLP, mm-IRMv1 consistently outperforms
    v-IRMv1. However, as the model scales up and becomes increasingly overparameterized (e.g.,
    ResNet-18), the non-smoothness introduced by mm-IRMv1 becomes more pronounced, allowing v-IRMv1
    to achieve superior performance.}
  \label{tab:wrap_dimensionality}
  \renewcommand{\arraystretch}{1.}
  \resizebox{0.5\textwidth}{!}{
    \begin{tabular}{lcc}
      \toprule
      \textbf{Methods} & \multicolumn{2}{c}{\textbf{Architecture (\#params)}} \\
      & 3-layer MLP ($\approx$0.3M) & ResNet18 ($\approx$11.7 M)\\
      \midrule
      IRMv1 & 65.7 $\pm$ 1.6 & 64.7 $\pm$ 0.5 \\
      \rowcolor{gray!15}v-IRMv1  &65.9 $\pm$ 1.6  & \textbf{68.1 $\pm$ 0.4 }\\
      \rowcolor{gray!15}mm-IRMv1 & \textbf{67.0 $\pm$ 1.0} & 66.8 $\pm$ 0.6 \\
      \bottomrule
    \end{tabular}
  }
  \vspace{-5mm}
\end{table}

\section{Conclusions}\label{ch:conclusion}

This paper demonstrated that IRMv1, the most widely used approximation of invariant risk
minimization (IRM), does not reliably achieve the intended invariance and is susceptible to
overfitting due to reliance on spurious features. To address this limitation, this work proposed a
novel approach that mitigates overfitting to limited training environments by extrapolating the
IRMv1 penalty across data distributions, thereby pseudo-diversifying the training environments. The
effectiveness of the proposed method was validated through both small-scale experiments using
structural equation models (SEMs) and evaluations under recent overparameterized settings. This
contribution offers a practical improvement to existing methods built upon the IRMv1 penalty and
lays the groundwork for more principled approximations of IRM.

\bibliography{main.bib}

\begin{thebibliography}{34}
\providecommand{\natexlab}[1]{#1}
\providecommand{\url}[1]{\texttt{#1}}
\expandafter\ifx\csname urlstyle\endcsname\relax
  \providecommand{\doi}[1]{doi: #1}\else
  \providecommand{\doi}{doi: \begingroup \urlstyle{rm}\Url}\fi

\bibitem[Ahuja et~al.(2020)Ahuja, Shanmugam, Varshney, and Dhurandhar]{b26}
K~Ahuja, K~Shanmugam, KR~Varshney, and A~Dhurandhar.
\newblock Invariant risk minimization games.
\newblock \emph{arXiv preprint arXiv:2002.04692}, 2020.

\bibitem[Ahuja et~al.(2022)Ahuja, Caballero, Zhang, Gagnon-Audet, Bengio, Mitliagkas, and Rish]{b23}
K~Ahuja, E~Caballero, D~Zhang, JC~Gagnon-Audet, Y~Bengio, I~Mitliagkas, and I~Rish.
\newblock Invariance principle meets information bottleneck for out-of-distribution generalization.
\newblock \emph{arXiv preprint arXiv:2106.06607}, 2022.

\bibitem[Arjovsky(2020)]{b2}
Martin Arjovsky.
\newblock \emph{Out of distribution generalization in machine learning}.
\newblock PhD thesis, New York University, 2020.

\bibitem[Arjovsky et~al.(2020)Arjovsky, Bottou, Gulrajani, and Lopez-Paz]{b11}
Martin Arjovsky, L{'e}on Bottou, Ishaan Gulrajani, and David Lopez-Paz.
\newblock Invariant risk minimization.
\newblock \emph{arXiv preprint arXiv:1907.02893}, 2020.

\bibitem[Armitage(2005)]{armitage2005encyclopedia}
Peter Armitage.
\newblock Encyclopedia of biostatistics.
\newblock \emph{(No Title)}, 7, 2005.

\bibitem[Ash \& Dol\'{e}ans-Dade(1999)Ash and Dol\'{e}ans-Dade]{ash.probability.book}
Robert~B Ash and Catherine Dol\'{e}ans-Dade.
\newblock \emph{Probability and Measure Theory}.
\newblock Academic Press, 2nd edition, 1999.

\bibitem[Ben-Tal et~al.(2013)Ben-Tal, Den~Hertog, De~Waegenaere, Melenberg, and Rennen]{ben2013robust}
Aharon Ben-Tal, Dick Den~Hertog, Anja De~Waegenaere, Bertrand Melenberg, and Gijs Rennen.
\newblock Robust solutions of optimization problems affected by uncertain probabilities.
\newblock \emph{Management Science}, 59\penalty0 (2):\penalty0 341--357, 2013.

\bibitem[Chen et~al.(2022)Chen, Zhou, Bian, Xie, Wu, Zhang, Ma, Yang, Zhao, Han, et~al.]{b24}
Y~Chen, K~Zhou, Y~Bian, B~Xie, B~Wu, Y~Zhang, K~Ma, H~Yang, P~Zhao, B~Han, et~al.
\newblock Pareto invariant risk minimization: Towards mitigating the optimization dilemma in out-of-distribution generalization.
\newblock \emph{arXiv preprint arXiv:2206.07766}, 2022.

\bibitem[Ganaie et~al.(2022)Ganaie, Hu, Malik, Tanveer, and Suganthan]{Ganaie_2022}
M.A. Ganaie, Minghui Hu, A.K. Malik, M.~Tanveer, and P.N. Suganthan.
\newblock Ensemble deep learning: A review.
\newblock \emph{Engineering Applications of Artificial Intelligence}, 115:\penalty0 105151, October 2022.
\newblock ISSN 0952-1976.
\newblock \doi{10.1016/j.engappai.2022.105151}.
\newblock URL \url{http://dx.doi.org/10.1016/j.engappai.2022.105151}.

\bibitem[Ganin et~al.(2016)Ganin, Ustinova, Ajakan, Germain, Larochelle, Laviolette, Marchand, and Lempitsky]{b36}
Yaroslav Ganin, Evgeniya Ustinova, Hana Ajakan, Pascal Germain, Hugo Larochelle, Fran{\c{c}}ois Laviolette, Mario Marchand, and Victor Lempitsky.
\newblock Domain-adversarial training of neural networks.
\newblock \emph{The journal of machine learning research}, 17\penalty0 (1):\penalty0 2096--2030, 2016.

\bibitem[Gulrajani \& Lopez-Paz(2020)Gulrajani and Lopez-Paz]{b29}
Ishaan Gulrajani and David Lopez-Paz.
\newblock In search of lost domain generalization.
\newblock \emph{arXiv preprint arXiv:2007.01434}, 2020.

\bibitem[He et~al.(2016)He, Zhang, Ren, and Sun]{he2016deep}
Kaiming He, Xiangyu Zhang, Shaoqing Ren, and Jian Sun.
\newblock Deep residual learning for image recognition.
\newblock In \emph{Proceedings of the IEEE conference on computer vision and pattern recognition}, pp.\  770--778, 2016.

\bibitem[Immer et~al.(2021)Immer, Bauer, Fortuin, R{\"a}tsch, and Emtiyaz]{immer2021scalable}
Alexander Immer, Matthias Bauer, Vincent Fortuin, Gunnar R{\"a}tsch, and Khan~Mohammad Emtiyaz.
\newblock Scalable marginal likelihood estimation for model selection in deep learning.
\newblock In \emph{International Conference on Machine Learning}, pp.\  4563--4573. PMLR, 2021.

\bibitem[Krogh \& Hertz(1991)Krogh and Hertz]{krogh1991simple}
Anders Krogh and John Hertz.
\newblock A simple weight decay can improve generalization.
\newblock \emph{Advances in neural information processing systems}, 4, 1991.

\bibitem[Krueger et~al.(2021)Krueger, Caballero, Jacobsen, Zhang, Binas, Zhang, Le~Priol, and Courville]{b25}
D~Krueger, E~Caballero, J-H Jacobsen, A~Zhang, J~Binas, D~Zhang, R~Le~Priol, and A~Courville.
\newblock Out-of-distribution generalization via risk extrapolation (rex).
\newblock \emph{arXiv preprint arXiv:2003.00688}, 2021.

\bibitem[Li et~al.(2017)Li, Yang, Song, and Hospedales]{b31}
Da~Li, Yongxin Yang, Yi-Zhe Song, and Timothy~M Hospedales.
\newblock Deeper, broader and artier domain generalization.
\newblock In \emph{Proceedings of the IEEE international conference on computer vision}, pp.\  5542--5550, 2017.

\bibitem[Lin et~al.(2022)Lin, Dong, Wang, and Zhang]{b27}
Y~Lin, H~Dong, H~Wang, and T~Zhang.
\newblock Bayesian invariant risk minimization.
\newblock In \emph{Proceedings of the IEEE/CVF Conference on Computer Vision and Pattern Recognition}, pp.\  16021--16030, 2022.

\bibitem[Minderer et~al.(2021)Minderer, Djolonga, Romijnders, Hubis, Zhai, Houlsby, Tran, and Lucic]{minderer2021revisiting}
Matthias Minderer, Josip Djolonga, Rob Romijnders, Frances Hubis, Xiaohua Zhai, Neil Houlsby, Dustin Tran, and Mario Lucic.
\newblock Revisiting the calibration of modern neural networks.
\newblock \emph{Advances in Neural Information Processing Systems}, 34:\penalty0 15682--15694, 2021.

\bibitem[Naganuma et~al.(2025)Naganuma, Hataya, Yoshida, and Mitliagkas]{naganuma2023empirical}
Hiroki Naganuma, Ryuichiro Hataya, Kotaro Yoshida, and Ioannis Mitliagkas.
\newblock An empirical study of pre-trained model selection for out-of-distribution generalization and calibration.
\newblock \emph{Transactions on Machine Learning Research}, 2025.
\newblock ISSN 2835-8856.
\newblock URL \url{https://openreview.net/forum?id=tYjoHjShxF}.

\bibitem[Nesterov(2004)]{nesterov.convex.opt.book}
Yurii Nesterov.
\newblock \emph{Introductory Lectures on Convex Optimization: A Basic Course}.
\newblock Kluwer Academic Publishers, 2004.

\bibitem[Nixon et~al.(2019)Nixon, Dusenberry, Zhang, Jerfel, and Tran]{nixon2019measuring}
Jeremy Nixon, Michael~W Dusenberry, Linchuan Zhang, Ghassen Jerfel, and Dustin Tran.
\newblock Measuring calibration in deep learning.
\newblock In \emph{CVPR workshops}, 2019.

\bibitem[Nocedal \& Wright(2006)Nocedal and Wright]{nocedal.wright.book}
Jorge Nocedal and Stephen~J Wright.
\newblock \emph{Numerical Optimization}.
\newblock Springer, 2nd edition, 2006.

\bibitem[Ovadia et~al.(2019)Ovadia, Fertig, Ren, Nado, Sculley, Nowozin, Dillon, Lakshminarayanan, and Snoek]{ovadia2019can}
Yaniv Ovadia, Emily Fertig, Jie Ren, Zachary Nado, David Sculley, Sebastian Nowozin, Joshua Dillon, Balaji Lakshminarayanan, and Jasper Snoek.
\newblock Can you trust your model's uncertainty? evaluating predictive uncertainty under dataset shift.
\newblock \emph{Advances in neural information processing systems}, 32, 2019.

\bibitem[Perez \& Wang(2017)Perez and Wang]{perez2017effectivenessdataaugmentationimage}
Luis Perez and Jason Wang.
\newblock The effectiveness of data augmentation in image classification using deep learning, 2017.
\newblock URL \url{https://arxiv.org/abs/1712.04621}.

\bibitem[Prechelt(2002)]{prechelt2002early}
Lutz Prechelt.
\newblock Early stopping-but when?
\newblock In \emph{Neural Networks: Tricks of the trade}, pp.\  55--69. Springer, 2002.

\bibitem[Qui{\~n}onero-Candela et~al.(2022)Qui{\~n}onero-Candela, Sugiyama, Schwaighofer, and Lawrence]{quinonero2022dataset}
Joaquin Qui{\~n}onero-Candela, Masashi Sugiyama, Anton Schwaighofer, and Neil~D Lawrence.
\newblock \emph{Dataset shift in machine learning}.
\newblock Mit Press, 2022.

\bibitem[Sagawa et~al.(2020)Sagawa, Koh, Hashimoto, and Liang]{sagawa2020distributionallyrobustneuralnetworks}
Shiori Sagawa, Pang~Wei Koh, Tatsunori~B. Hashimoto, and Percy Liang.
\newblock Distributionally robust neural networks for group shifts: On the importance of regularization for worst-case generalization, 2020.
\newblock URL \url{https://arxiv.org/abs/1911.08731}.

\bibitem[Torralba \& Efros(2011)Torralba and Efros]{torralba2011unbiased}
Antonio Torralba and Alexei~A Efros.
\newblock Unbiased look at dataset bias.
\newblock In \emph{CVPR 2011}, pp.\  1521--1528. IEEE, 2011.

\bibitem[Vapnik(1991)]{b1}
Vladimir Vapnik.
\newblock Principles of risk minimization for learning theory.
\newblock \emph{Advances in neural information processing systems}, 4, 1991.

\bibitem[Wald et~al.(2021)Wald, Feder, Greenfeld, and Shalit]{wald2021on}
Yoav Wald, Amir Feder, Daniel Greenfeld, and Uri Shalit.
\newblock On calibration and out-of-domain generalization.
\newblock In A.~Beygelzimer, Y.~Dauphin, P.~Liang, and J.~Wortman Vaughan (eds.), \emph{Advances in Neural Information Processing Systems}, 2021.
\newblock URL \url{https://openreview.net/forum?id=XWYJ25-yTRS}.

\bibitem[Ye et~al.(2022)Ye, Li, Bai, Yu, Hong, Zhou, Li, and Zhu]{ye2022ood}
Nanyang Ye, Kaican Li, Haoyue Bai, Runpeng Yu, Lanqing Hong, Fengwei Zhou, Zhenguo Li, and Jun Zhu.
\newblock Ood-bench: Quantifying and understanding two dimensions of out-of-distribution generalization.
\newblock In \emph{Proceedings of the IEEE/CVF Conference on Computer Vision and Pattern Recognition}, pp.\  7947--7958, 2022.

\bibitem[Yoshida \& Naganuma(2024)Yoshida and Naganuma]{yoshida2024towards}
Kotaro Yoshida and Hiroki Naganuma.
\newblock Towards understanding variants of invariant risk minimization through the lens of calibration.
\newblock \emph{Transactions on Machine Learning Research}, 2024.
\newblock ISSN 2835-8856.
\newblock URL \url{https://openreview.net/forum?id=9YqacugDER}.

\bibitem[Zhang et~al.(2023)Zhang, Sharma, Ram, Hong, Varshney, and Liu]{b37}
Yihua Zhang, Pranay Sharma, Parikshit Ram, Mingyi Hong, Kush Varshney, and Sijia Liu.
\newblock What is missing in {IRM} training and evaluation? challenges and solutions.
\newblock \emph{arXiv preprint arXiv:2303.02343}, 2023.

\bibitem[Zhou et~al.(2022)Zhou, Lin, Zhang, and Zhang]{pmlr-v162-zhou22e}
Xiao Zhou, Yong Lin, Weizhong Zhang, and Tong Zhang.
\newblock Sparse invariant risk minimization.
\newblock In Kamalika Chaudhuri, Stefanie Jegelka, Le~Song, Csaba Szepesvari, Gang Niu, and Sivan Sabato (eds.), \emph{Proceedings of the 39th International Conference on Machine Learning}, volume 162 of \emph{Proceedings of Machine Learning Research}, pp.\  27222--27244. PMLR, 17--23 Jul 2022.
\newblock URL \url{https://proceedings.mlr.press/v162/zhou22e.html}.

\end{thebibliography}
\bibliographystyle{tmlr}

\appendix
\newpage

\section*{Appendices}

\section{Proof of Theorem 3.1}\label{ch:appendix1}

First, by the definition of $\mathcal{F}_{\delta}$,
\begin{align*}
  \sum\nolimits_{e \in E_{\text{train}}} R_e (\pi \cdot \Phi) \leq \delta\,, \quad \forall (\pi,
  \Phi) \in \mathcal{F}_{\delta} \,,
\end{align*}
and the initial hypothesis that $R_e(\cdot) \geq 0$, it can be easily concluded that
\begin{align*}
  R_e (\pi \cdot \Phi) \le \delta\,, \quad \forall (\pi, \Phi)\in \mathcal{F}_{\delta}\,, \forall
  e \in E_{\mathrm{train}} \,.
\end{align*}

Assume, for contradiction to the claim of the theorem, that there exist an $\hat{e} \in
E_{\text{train}}$ and $(\hat{\pi}, \hat{\Phi}) \in \mathcal{F}_{\delta}$ such that
\begin{align}
  \lvert \nabla_{ \pi } R_{\hat{e}} (\hat{\pi} \cdot \hat{\Phi}) \rvert
  > \sqrt{2L_{\hat{\Phi}} \delta} \,. \label{assume.contra}
\end{align}
Hereafter, to simplify notation, $R_{\hat{e}} (\hat{\pi} \cdot \hat{\Phi})$ will be expressed as
$R_{\hat{e}} (\hat{\pi})$ to stress the fact that it is a function of $\hat{\pi}$ since
$\hat{\Phi}$ is fixed.

Let $\hat{\pi}' \coloneqq \hat{\pi} - \eta \nabla_{\pi} R_{\hat{e}} (\hat{\pi} \cdot
\hat{\Phi})$. By \Cref{as:l-smooth}, the descent lemma~\citep[Lem.~1.2.3]{nesterov.convex.opt.book}
suggests
\begin{align}
  R_{\hat{e}} (\hat{ \pi}')
  & = R_{\hat{e}} (\, \hat{\pi} - \eta \nabla_{\pi} R_{\hat{e}} (\hat{\pi})\, ) \notag \\
  & \leq R_{\hat{e}} (\hat{\pi}) + \nabla_{\pi} R_{\hat{e}} (\hat{ \pi})^{\top}
  (\, -\eta \nabla_{\pi} R_{\hat{e}} (\hat{\pi})\, ) + \frac{L_{\hat{\Phi}}}{2} \lvert -\eta
  \nabla_{\pi} R_{\hat{e}} (\hat{\pi}) \rvert^2 \notag \\
  & = R_{\hat{e}} (\hat{\pi}) - \eta \lvert \nabla_{\pi} R_{\hat{e}} (\hat{\pi}) \rvert^2
  + \frac{L_{\hat{\Phi}}}{2} \eta^2 \lvert \nabla_{\pi} R_{\hat{e}} (\hat{\pi}) \rvert^2 \notag \\
  & = R_{\hat{e}} (\hat{\pi}) - (\eta - \frac{L_{\hat{\Phi}}}{2} \eta^2 )\,
  \lvert \nabla R_{\hat{e}} (\hat{\pi}) \rvert^2 \,. \label{descent.lemma}
\end{align}

Let $\eta = 1/L_{\hat{\Phi}}$. Then
\begin{align*}
  \eta - \frac{L_{\hat{\Phi}}}{2} \eta^2 = \frac{1}{L_{\hat{\Phi}}} - \frac{L_{\hat{\Phi}}}{2}
  \frac{1}{L_{\hat{\Phi}}^2} = \frac{1}{L_{\hat{\Phi}}} - \frac{1}{2L_{\hat{\Phi}}} =
  \frac{1}{2L_{\hat{\Phi}}} \,.
\end{align*}
Hence, by (\ref{assume.contra}) and (\ref{descent.lemma}),
\begin{align*}
  R_{\hat{e}} (\hat{\pi}') & \leq R_{\hat{e}} (\hat{\pi}) - \frac{1}{2L_{\hat{\Phi}}} \lvert
  \nabla_{\pi} R_{\hat{e}} (\hat{\pi}) \rvert^2 \\
  & < \delta - \frac{1}{2L_{\hat{\Phi}}} 2L_{\hat{\Phi}} \delta \\
  & = \delta - \delta = 0 \,,
\end{align*}
which contradicts the initial hypothesis that $R_e(\cdot) \geq 0$. This completes the proof.

\section{Proof of Lemma 4.1}\label{sec:proof_lemma41}

Notice that $\forall (\pi, \Phi)$,
\begin{align*}
  \lvert \nabla_{\pi} R_e (\pi \cdot \Phi) \rvert^2
  & = \lvert\, \nabla_{ \pi}\, \mathbb{E}_{(x,y) \sim P_e } \left\{\, \ell ( \pi \cdot\Phi(x),y)\,
  \right\}\, \rvert^2 \\
  & = \lvert\, \mathbb{E}_{(x,y) \sim P_e } \left\{\, \nabla_{ \pi} \ell ( \pi \cdot\Phi(x),y)\,
  \right\}\, \rvert^2 \\
  & \leq \mathbb{E}_{(x,y) \sim P_e} \left\{\, \lvert \nabla_{ \pi } \ell ( \pi \cdot\Phi(x),y)
  \rvert^2\, \right\} \\
  & = \mathcal{J}_{ \textnormal{IRM}, e} (\pi, \Phi) \,,
\end{align*}
where the interchange between the expectation and the partial gradient in the second equality
follows from the dominated convergence theorem, while the inequality is a direct application of
Jensen's inequality \citep{ash.probability.book}.

\section{Derivation of the closed-form expression of (\ref{eq:mm-irmv1})
} \label{ch:appendix_mm_proof}

Problem (\ref{eq:mm-irmv1}) belongs to the class of linear-programming (LP) problems
\citep[(13.1)]{nocedal.wright.book}. By defining $\bm{\mu} \coloneqq (\mu_e)_{e \in
  E_{\text{train}}} \in \mathbb{R}_+^{ \lvert E_{\text{train}} \rvert}$ and for a $\varrho \in
\mathbb{R}$, the Lagrangian function \citep[(13.3)]{nocedal.wright.book} for (\ref{eq:mm-irmv1})
becomes:
\begin{align*}
  \mathcal{L} ( \bm{\alpha}, \varrho, \bm{\mu} )
  = - \sum\nolimits_{e \in E_{\text{train}} } \alpha_e\, \mathcal{J}_{\textnormal{IRMv1}, e} (\Phi)
  + \varrho \left( 1 - \sum\nolimits_{e \in E_{\text{train}} } \alpha_e \right)
  + \sum\nolimits_{e \in E_{\text{train}} } \mu_e (\alpha_{\min} - \alpha_e) \,,
\end{align*}
where the minus sign is applied to $\mathcal{J}_{\textnormal{IRMv1}, e} (\Phi)$ to align with the
convention that the Lagrangian function is typically formulated for minimization. It is well-known
for LP problems that $\bm{\alpha}$ solves (\ref{eq:mm-irmv1}) if and only if there exist $(\varrho,
\bm{\mu} )$ such that the following Karush-Kuhn-Tucker (KKT) conditions are satisfied
\citep[(13.4)]{nocedal.wright.book}:
\begin{subequations}
  \begin{align}
    & \text{(Stationarity:)} && \frac{\partial \mathcal{L}}{\partial \alpha_e}
    = - \mathcal{J}_{\textnormal{IRMv1}, e}(\Phi) - \varrho - \mu_e = 0\,,
    \quad\forall e\in E_{\textnormal{train}} \,, \label{kkt.stationarity}\\
    & \text{(Complementarity condition:)}
    && \mu_e\,(\alpha_{\min} - \alpha_e) = 0\,,
    \quad \forall e \in E_{\textnormal{train}}\,, \label{kkt.complementarity}\\
    & \text{(Primal feasibility:)}
    && \sum\nolimits_{e \in E_{\textnormal{train}} } \alpha_e = 1\,, \quad \alpha_e \ge
    \alpha_{\min}\,, \label{kkt.primal.feas}\\
    & \text{(Dual feasibility:)}
    && \mu_e \geq 0\,, \quad\forall e \in E_{\textnormal{train}} \,. \label{kkt.dual.feas}
  \end{align}
\end{subequations}
Condition (\ref{kkt.primal.feas}) suggests $1 = \sum_e \alpha_e \geq \sum_e \alpha_{\min} =
\alpha_{\min}\, \lvert E_{ \textnormal{train} } \rvert \Rightarrow \alpha_{\min} \leq 1 / \lvert E_{
  \textnormal{train} } \rvert$. Consequently, the following two cases are considered.

\noindent\textbf{Case 1:} $\alpha_{\min} = 1 / \lvert E_{ \textnormal{train} } \rvert$. If there
exists an $e\in E_{\textnormal{train}}$ such that $\alpha_e > \alpha_{\min} = 1 / \lvert E_{
  \textnormal{train} } \rvert$, then (\ref{kkt.primal.feas}) yields $1 = \sum_e \alpha_e > \sum_e
\alpha_{\min} = \sum_e 1 / \lvert E_{ \textnormal{train} } \rvert = 1$, which is absurd. Hence,
$\alpha_e = 1 / \lvert E_{ \textnormal{train} } \rvert$, $\forall e \in E_{\textnormal{train}}$. In
such a case,
\begin{align}
  C_{\text{mm}} (\Phi) = \frac{1}{ \lvert E_{ \textnormal{train} } \rvert } \sum\nolimits_{ e
    \in E_{ \textnormal{train} } } \mathcal{J}_{\text{IRMv1}, e} (\Phi) \,. \label{Cmm.case1}
\end{align}

\noindent\textbf{Case 2:} $\alpha_{\min} < 1 / \lvert E_{ \textnormal{train} } \rvert$. If $\alpha_e
= \alpha_{\min}$, $\forall e \in E_{ \textnormal{train} }$, then the absurd result $1 = \sum_e
\alpha_e = \sum_e \alpha_{\min} < \sum_e 1 / \lvert E_{ \textnormal{train} } \rvert = 1$ is
obtained, which suggests that there must exist at least one $e \in E_{ \textnormal{train} }$ such
that $\alpha_e > \alpha_{\min}$.

\begin{enumerate}[ label = \textbf{(\roman*)}, itemsep = 0pt, topsep = 0pt ]

\item Consider those $e \in E_{ \textnormal{train} }$ with $\alpha_e > \alpha_{\min}$---there
  certainly exists at least one such $e$. Then, (\ref{kkt.complementarity}) yields $\mu_e = 0$, and
  (\ref{kkt.stationarity}) leads to $\mathcal{J}_{\textnormal{IRMv1}, e}(\Phi) = - \varrho$.

\item Consider those $e \in E_{ \textnormal{train} }$ with $\alpha_e = \alpha_{\min}$. Then, by
  (\ref{kkt.stationarity}) and (\ref{kkt.dual.feas}), $\mu_e = - \mathcal{J}_{\textnormal{IRMv1}, e}
  (\Phi) - \varrho \geq 0 \Rightarrow \mathcal{J}_{\textnormal{IRMv1}, e} (\Phi) \leq - \varrho$.

\end{enumerate}

The previous two points suggest $-\varrho = \max_{ e\in E_{ \textnormal{train} } }
\mathcal{J}_{\textnormal{IRMv1}, e}(\Phi)$. Upon defining $E_{\max} \coloneqq \arg\max_{ e\in E_{
    \textnormal{train} } } \mathcal{J}_{\textnormal{IRMv1}, e}(\Phi)$ and its complement
$E_{\max}^{\complement} \coloneqq E_{\textnormal{train}} \setminus E_{\max}$, it becomes also clear
that if $e\in E_{\max}^{\complement}$, that is $\mathcal{J}_{\textnormal{IRMv1}, e} (\Phi) < -
\varrho$, then necessarily $\alpha_e = \alpha_{\min}$. Because such an $\bm{\alpha}$ satisfies the
KKT conditions, plugging it back into (\ref{eq:mm-irmv1}) yields
\begin{alignat*}{2}
  \mathcal{C}_{ \text{mm} }(\Phi)
  & {} = {} && \sum\nolimits_{ e \in E_{\text{train}} } \alpha_e\, \mathcal{J}_{\text{IRMv1}, e} (\Phi)
  \notag\\
  & = && \sum\nolimits_{ e \in E_{ \max } } \alpha_e\, \max\nolimits_{e \in E_{\text{train}} }
  \mathcal{J}_{\text{IRMv1}, e} (\Phi) + \sum\nolimits_{ e \in E_{ \max }^{\complement} }
  \alpha_{\min}\, \mathcal{J}_{\text{IRMv1}, e} (\Phi) \\
  & = && \left( 1 - \sum\nolimits_{ e \in E_{ \max }^{\complement} } \alpha_{\min} \right )
  \max\nolimits_{e \in E_{\text{train}} } \mathcal{J}_{\text{IRMv1}, e} (\Phi) + \alpha_{\min}
  \sum\nolimits_{ e \in E_{ \max }^{\complement} } \mathcal{J}_{\text{IRMv1}, e} (\Phi) \\
  & = && \left ( 1 - \alpha_{\min} \lvert E_{ \max }^{\complement} \rvert \right ) \max\nolimits_{e
    \in E_{\text{train}} } \mathcal{J}_{\text{IRMv1}, e} (\Phi) + \alpha_{\min} \sum\nolimits_{ e
    \in E_{ \textnormal{train} } } \mathcal{J}_{\text{IRMv1}, e} (\Phi) \\
  &&& - \alpha_{\min} \sum\nolimits_{ e \in E_{ \max } } \mathcal{J}_{\text{IRMv1}, e} (\Phi) \\
  & = && \left ( 1 - \alpha_{\min} \lvert E_{ \max }^{\complement} \rvert \right ) \max\nolimits_{ e
    \in E_{\text{train}} } \mathcal{J}_{\text{IRMv1}, e} (\Phi) + \alpha_{\min} \sum\nolimits_{ e
    \in E_{ \textnormal{train} } } \mathcal{J}_{\text{IRMv1}, e} (\Phi) \\
  &&& - \alpha_{\min}\, \lvert E_{ \max } \rvert\, \max\nolimits_{e \in E_{\text{train}} }
  \mathcal{J}_{\text{IRMv1}, e} (\Phi) \\
  & = && \left ( 1 - \alpha_{\min} \lvert E_{ \textnormal{train} } \rvert \right ) \max\nolimits_{ e
    \in E_{\text{train}} } \mathcal{J}_{\text{IRMv1}, e} (\Phi) + \alpha_{\min} \sum\nolimits_{ e
    \in E_{ \textnormal{train} } } \mathcal{J}_{\text{IRMv1}, e} (\Phi) \,,
\end{alignat*}
where $\lvert E_{ \textnormal{train} } \rvert = \lvert E_{ \max } \rvert + \lvert E_{ \max
}^{\complement} \rvert$ was used to deduce the last equality. Notice that this last expression for
$C_{\text{mm}}$ covers also the case $\alpha_{\min} = 1 / \lvert E_{ \textnormal{train} } \rvert$ of
(\ref{Cmm.case1}), establishing thus (\ref{eq:mm-irmv1}).

\section{Implementation Details}\label{ch:appendix2}

\subsection{SEMs}
\label{subsec:sem_settings}

All experiments are conducted on a single NVIDIA Tesla T4 GPU.

The experimental setup follows that of \citet{b11}. Each environment contains 1,000 generated
samples, and training is performed using full-batch updates. The number of training iterations is
fixed at \num{20000}, with a constant learning rate of \num{1e-3}. Penalty-related scaling factors
are selected via grid search. The search ranges are: ${1 \times 10^0, 1 \times 10^1}$ for
$\lambda$, ${-1 \times 10^0, -5 \times 10^0, -1 \times 10^1}$ for $\alpha_{\text{min}}$, and
${1 \times 10^0, 1 \times 10^1, 1 \times 10^2}$ for $\gamma$. Results are averaged over three
random seeds, and standard deviations are reported.

\subsection{Vision Datasets}

All experiments are conducted using a single NVIDIA H100 GPU.

The datasets and their respective configurations are described as follows.

\begin{itemize}
\item \textbf{Colored MNIST (Colored FashionMNIST):} Colored MNIST (and Colored FashionMNIST) is a
  binary classification dataset derived from MNIST (or FashionMNIST), where digits 0–4 are assigned
  to class 0 and digits 5–9 to class 1. Each dataset consists of \num{70000} samples (\num{50000}
  for training and \num{20000} for testing), with input dimensions of $(1, 28, 28)$. The invariant
  features correspond to the original grayscale images, while the spurious features are encoded as a
  single-channel color (e.g., red or green) applied to each image. Typically, each class is
  associated with a specific color, but this color flips for a proportion $p$ of the samples. In the
  experiments, the training environments use $p = [10\%, 20\%]$, and the test environment uses $p =
  [90\%]$.

\item \textbf{PACS:} The PACS dataset contains \num{9991} images of size $(3, 224, 224)$ across 7
  classes and 4 distinct visual styles, each corresponding to a different domain. In the
  experiments, the Cartoon, Photo, and Sketch domains are used for training, while the Art Painting
  domain is used for testing.

\item \textbf{VLCS:} The VLCS dataset consists of \num{10729} images of size $(3, 224, 224)$ across
  5 classes, drawn from four different sources. For the experimental setup, the Caltech101, LabelMe,
  and VOC2007 domains are used for training, while SUN09 serves as the test domain.
\end{itemize}

The Adam optimizer is used with a learning rate of $5 \times 10^{-4}$ across all datasets. For
correlation shift experiments, models are trained using full-batch optimization for \num{500}
epochs. For other settings, training is performed for \num{200} epochs with a batch size of
\num{32}. Following the protocol of \citet{b27, b37}, the penalty coefficient $\lambda$ is set to $1
\times 10^6$ for correlation shift experiments and $1 \times 10^0$ for diversity shift experiments.


The hyperparameters $\alpha_{\text{min}}$ and $\gamma$ are tuned via grid search. The search ranges
are:
\begin{itemize}
\item $\alpha_{\text{min}} \in \{-1 \times 10^{-1}, -2 \times 10^{-1}, \ldots, -1 \times 10^0\}$
\item $\gamma \in \{1 \times 10^{-1}, 2 \times 10^{-1}, \ldots, 1 \times 10^0\}$
\end{itemize}

Model selection follows the ``test-domain validation set'' strategy \citep{b29} for correlation
shift settings, and the ``training-domain validation set'' strategy \citep{b29} for diversity shift
settings. All reported results are averaged over three random seeds, and standard deviations are
also computed.

\section{Additional Results}\label{ch:appendix3}

\subsection{Results on SEMs}

\begin{table}[ht]
  \caption{Invariance errors in SEMs. Even in scenarios where the training environments are
    similar---making it difficult to eliminate spurious features---the proposed methods,
    particularly mm-IRMv1, consistently achieve substantial improvements over the IRMv1
    baseline.}
  \renewcommand{\arraystretch}{1.2} \resizebox{\textwidth}{!}{%
    \begin{tabular}{lcc|cc|cc}
      \toprule
      &\multicolumn{2}{c|}{$E_{\text{train}} = \ $\{0.1, 0.5, 1\}} &
      \multicolumn{2}{c|}{$E_{\text{train}} = \ $\{0.1, 0.4, 0.7, 1\}} &
      \multicolumn{2}{c}{$E_{\text{train}} = \ $\{0.1, 0.2, 0.3, 0.4, 0.5\}}
      \\
      & causal err ($\downarrow$) & non-causal err ($\downarrow$)
      & causal err ($\downarrow$) & non-causal err ($\downarrow$)
      & causal err ($\downarrow$) & non-causal err ($\downarrow$) \\
      \midrule
      IRMv1 & 0.679 $\pm$ 0.093 & 0.414 $\pm$ 0.046 & 0.986 $\pm$ 0.513 & 0.524 $\pm$ 0.192 & 1.179
      $\pm$ 0.424 & 0.570 $\pm$ 0.231 \\
      \rowcolor{gray!15} v-IRMv1 (Ours) & \textbf{0.597 $\pm$ 0.024} & \textbf{0.381 $\pm$ 0.014} &
      \textbf{0.689 $\pm$ 0.320} & \textbf{0.402 $\pm$ 0.094} & \textbf{0.776 $\pm$ 0.567} &
      \textbf{0.330 $\pm$ 0.195} \\
      \rowcolor{gray!15} mm-IRMv1 (Ours) & \textbf{0.427 $\pm$ 0.220} & \textbf{0.305 $\pm$ 0.146} &
      \textbf{0.527 $\pm$ 0.252} & \textbf{0.273 $\pm$ 0.127} & \textbf{0.736 $\pm$ 0.537} &
      \textbf{0.373 $\pm$ 0.204} \\
      \bottomrule
  \end{tabular}}
  \label{tab:sem_result_others}
\end{table}

In \Cref{tab:sem_result}, the number of training environments is limited to two. Additional
experiments are conducted under alternative settings, with results presented in
\Cref{tab:sem_result_others}. Specifically, scenarios with three, four, and five training
environments are evaluated. Across all cases, our methods consistently achieves substantial
improvements.


\subsection{Relationship between IRM penalty and Test
  Performance} \label{sec:additional_penalty_relation}

This section provides additional figures that illustrate the effectiveness of the proposed methods
in mitigating overfitting of the IRM penalty, consistent with the results shown in
\Cref{fig:vertical_images}.

\subsubsection{BIRM}

Figure~\ref{fig:birm_penalty} illustrates the impact of the proposed extrapolation method on BIRM
for CMNIST. While all methods perform similarly in minimizing $\mathcal{J}_{\text{IRMv1}, e}(\Phi)$
within the training environments, notable differences emerge in test performance, demonstrating that
the proposed approach effectively mitigates overfitting to the training data. Furthermore, although
the test performance of the original BIRM remains nearly constant, it exhibits instability in
$\mathcal{J}_{\text{IRMv1}, e}(\Phi)$ across the training environments.

\begin{figure}[ht]
  \vspace{-3mm}
  \centering
  \begin{subfigure}[b]{0.33\linewidth}
    \includegraphics[width=\linewidth]{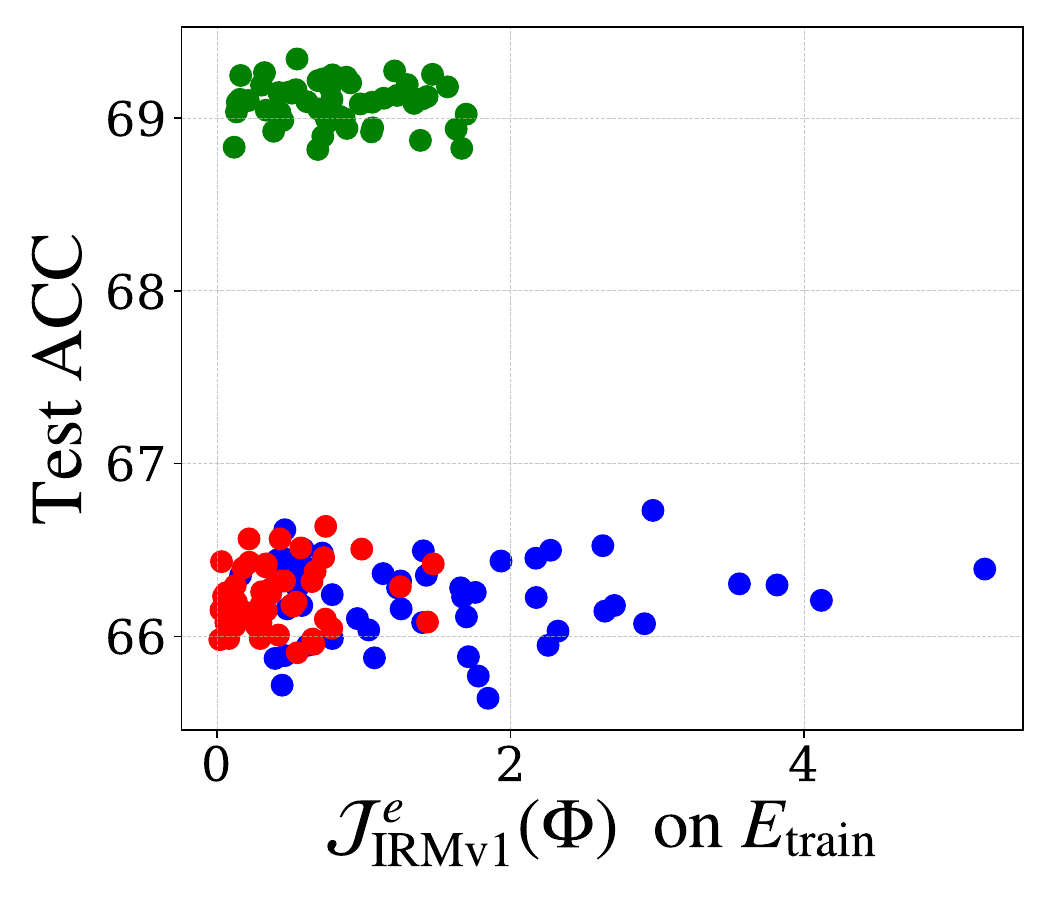}
    \label{fig:cmnist_extend2}
  \end{subfigure}%
  \begin{subfigure}[b]{0.33\linewidth}
    \includegraphics[width=\linewidth]{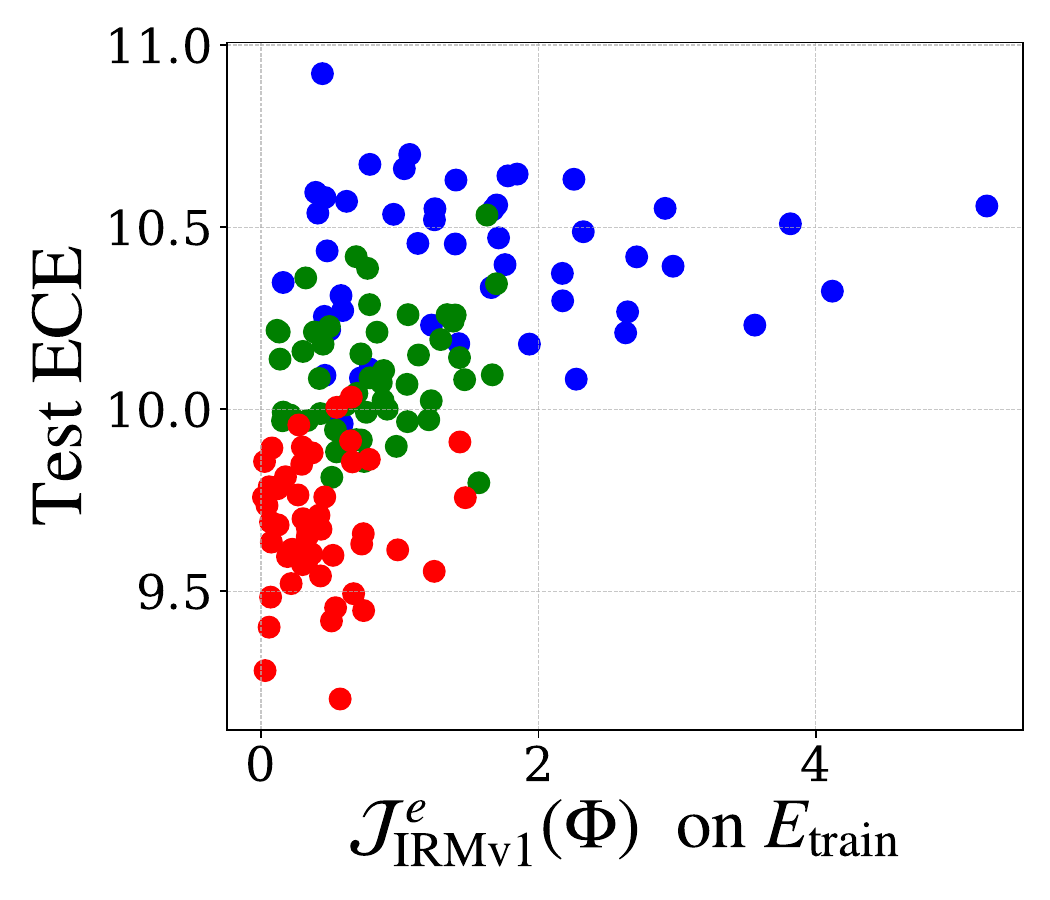}
    \label{fig:cobj_grad3}
  \end{subfigure}%
  \begin{subfigure}[b]{0.33\linewidth}
    \includegraphics[width=\linewidth]{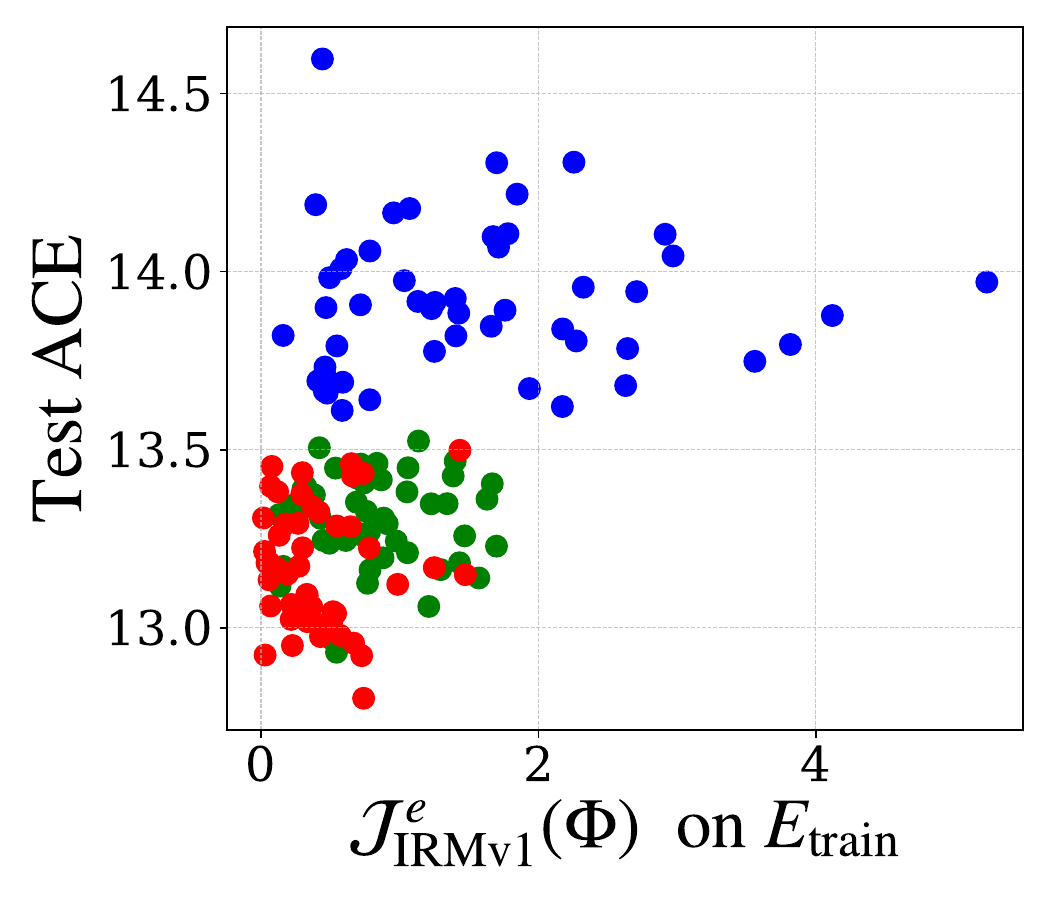}
    \label{fig:cobj_grad4}
  \end{subfigure}%
  \vspace{-5mm}
  \begin{center}
    \begin{tabular}{clclcl}
      \tikz \draw[blue,fill=blue] (0,0) circle (0.5ex); & BIRM&
      \tikz \draw[matplotlibgreen,fill=matplotlibgreen] (0,0) circle (0.5ex); & v-BIRM (Ours)&
      \tikz \draw[red,fill=red](0,0)  circle (0.5ex); & mm-BIRM (Ours)\\
    \end{tabular}
  \end{center}

  \caption{ The relationship between the IRMv1 penalty values in the training environment and the
    corresponding test evaluation metrics (from left to right: accuracy, ECE, and ACE) on
    CMNIST. Each point represents the values recorded at each epoch during the last 50 epochs for
    each method. \TODO{Add analysis} }
  \label{fig:birm_penalty}
  \vspace{-2mm}
\end{figure}

\subsubsection{BLO}

Figure \ref{fig:blo_penalty} illustrates the impact of the proposed extrapolation method on BLO
for CMNIST. While all methods achieve comparable values of $\mathcal{J}_{\text{IRMv1}, e}(\Phi)$ in
the training environments, substantial differences arise in test performance, indicating that the
proposed approach effectively mitigates overfitting to the training data. Notably, both proposed
methods improve calibration metrics relative to the original BLO, demonstrating that the
distributional extrapolation technique reduces overconfidence.


\begin{figure}[ht]
  \vspace{-3mm}
  \centering
  \begin{subfigure}[b]{0.33\linewidth}
    \includegraphics[width=\linewidth]{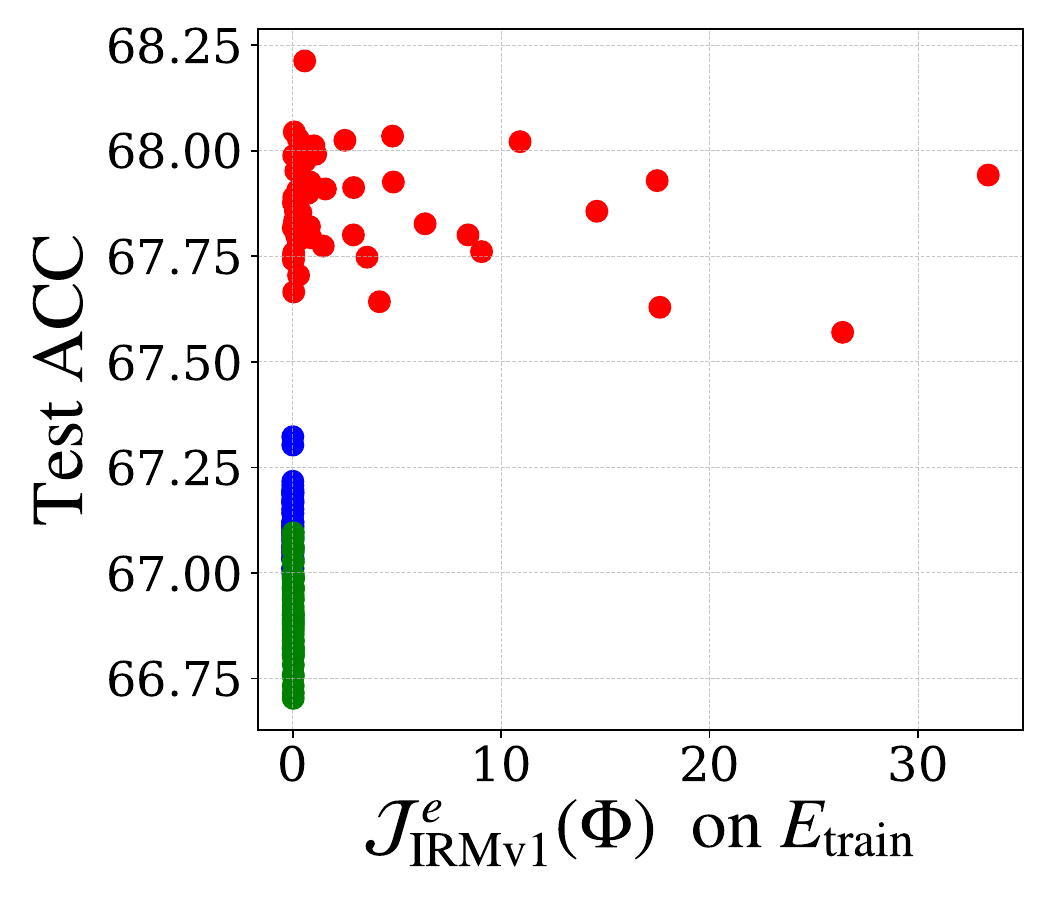}
    \label{fig:cmnist_extend3}
  \end{subfigure}%
  \begin{subfigure}[b]{0.33\linewidth}
    \includegraphics[width=\linewidth]{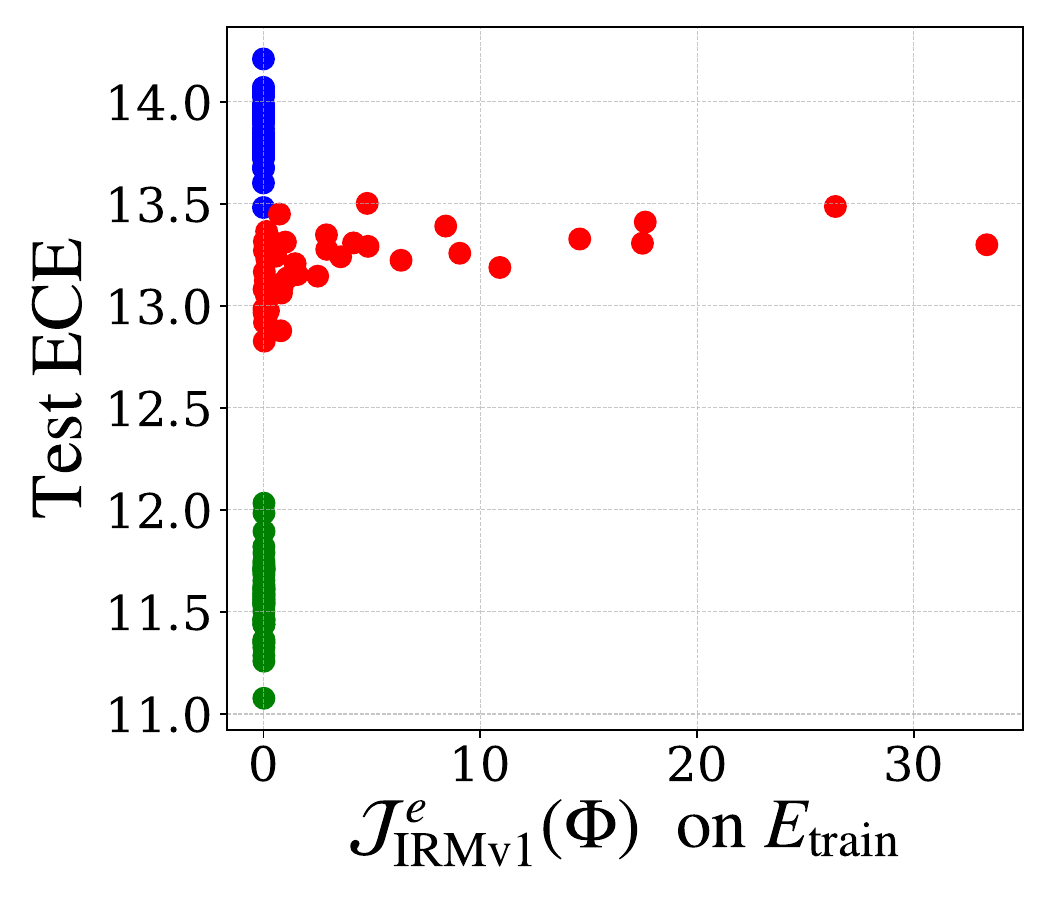}
    \label{fig:cobj_grad5}
  \end{subfigure}%
  \begin{subfigure}[b]{0.33\linewidth}
    \includegraphics[width=\linewidth]{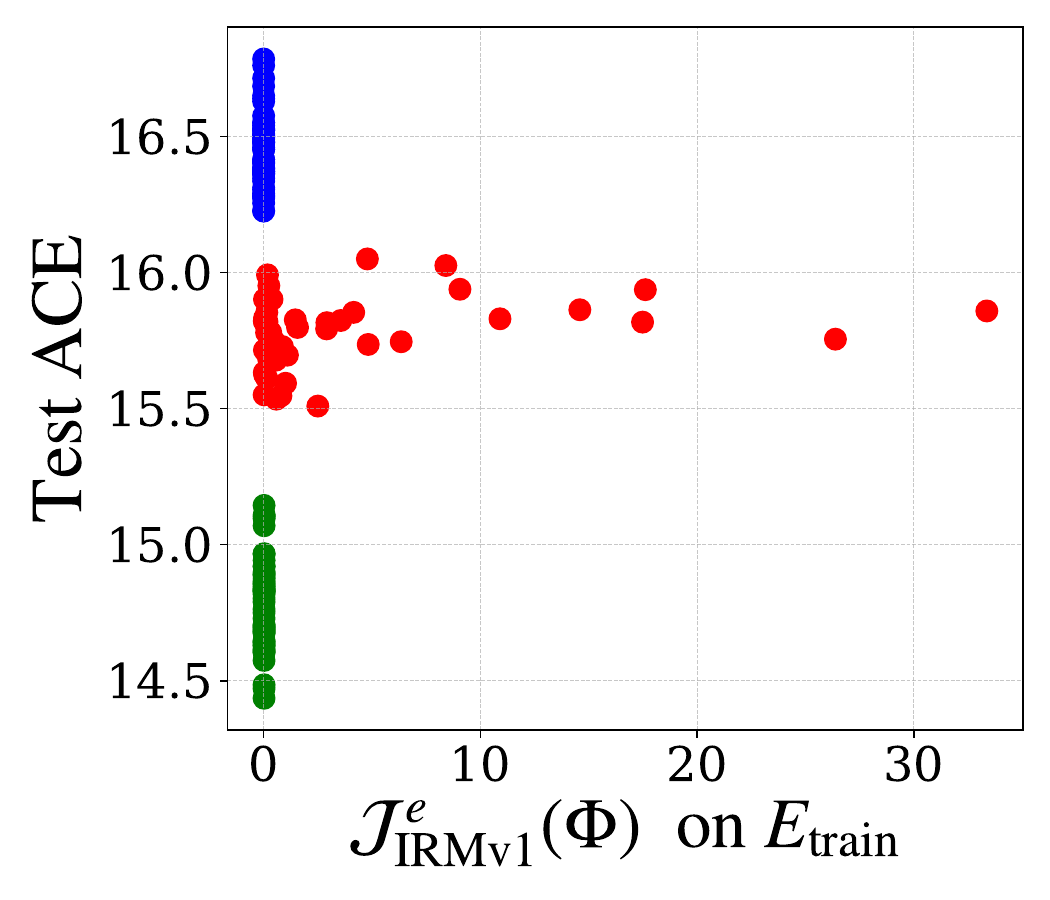}
    \label{fig:cobj_grad6}
  \end{subfigure}%
  \vspace{-5mm}
  \begin{center}
    \begin{tabular}{clclcl}
      \tikz \draw[blue,fill=blue] (0,0) circle (0.5ex); & BLO&
      \tikz \draw[matplotlibgreen,fill=matplotlibgreen] (0,0) circle (0.5ex); & v-BLO (Ours)&
      \tikz \draw[red,fill=red](0,0)  circle (0.5ex); & mm-BLO (Ours)\\
    \end{tabular}
  \end{center}

  \caption{ The relationship between the IRM penalty values in the training environment and the
    corresponding test evaluation metrics (from left to right: accuracy, ECE, and ACE) on
    CMNIST. Each point represents the values recorded at each epoch during the last 50 epochs for
    each method. \TODO{Add analysis} }
  \label{fig:blo_penalty}
  \vspace{-2mm}
\end{figure}

\subsubsection{Extrapolated IRMv1 and other existing methods}

Figure \ref{fig:all_penalty} presents a comparison of v-IRMv1, mm-IRMv1, and other existing methods
on CMNIST. Notably, v-IRMv1 maintains a consistently low $\mathcal{J}_{\text{IRMv1}, e}(\Phi)$ in
the training environments while achieving superior average test accuracy and calibration metrics
compared to the other methods. This suggests that invariant learning is effectively achieved. In
contrast, although IRMv1 reports relatively low calibration errors, it exhibits the lowest test
accuracy overall. Additionally, both BIRM and BLO demonstrate suboptimal performance on at least one
evaluation metric.

\begin{figure}[ht]
  \vspace{-3mm}
  \centering
  \begin{subfigure}[b]{0.33\linewidth}
    \includegraphics[width=\linewidth]{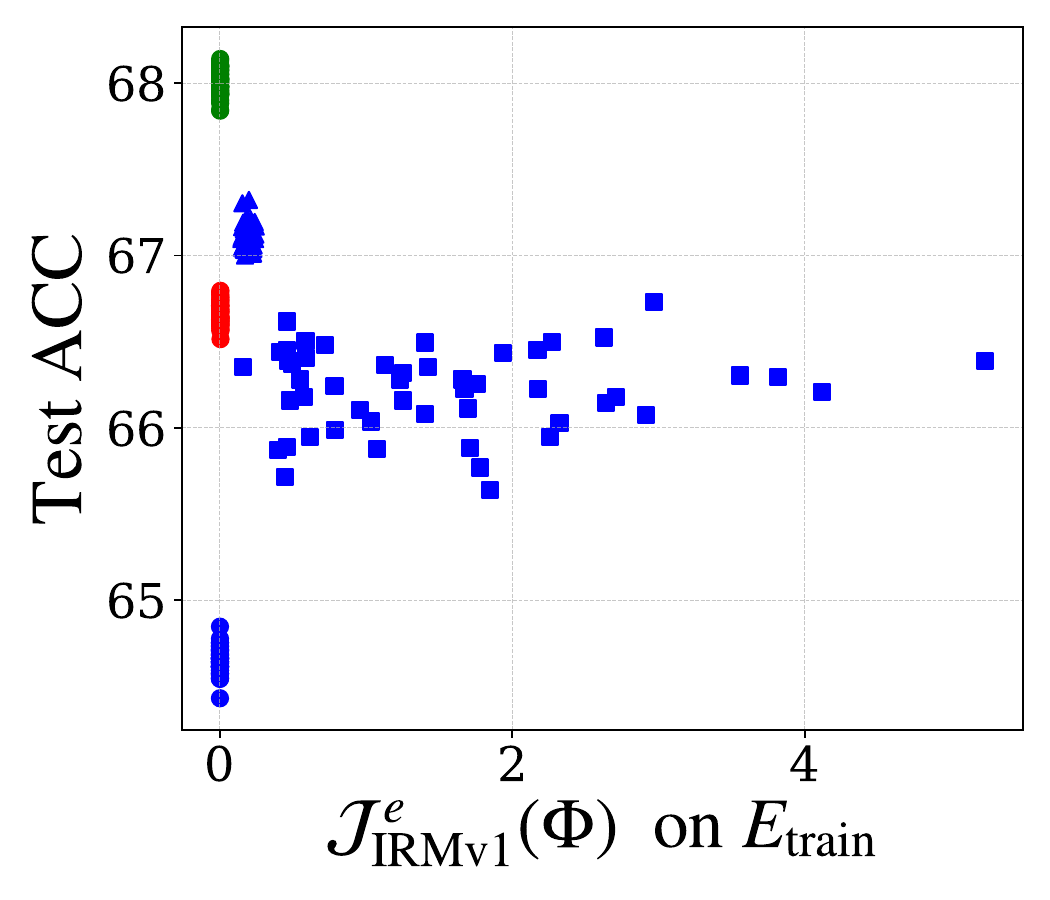}
    \label{fig:cmnist_extend4}
  \end{subfigure}%
  \begin{subfigure}[b]{0.33\linewidth}
    \includegraphics[width=\linewidth]{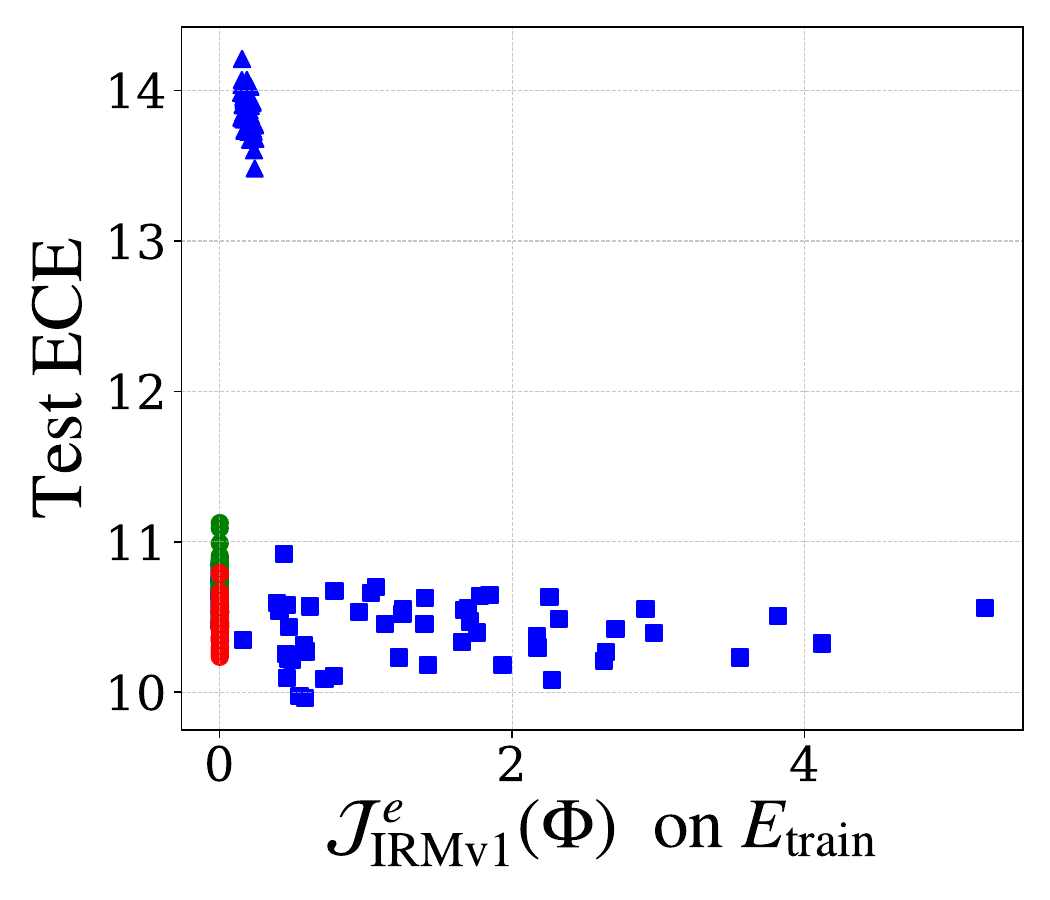}
    \label{fig:cobj_grad7}
  \end{subfigure}%
  \begin{subfigure}[b]{0.33\linewidth}
    \includegraphics[width=\linewidth]{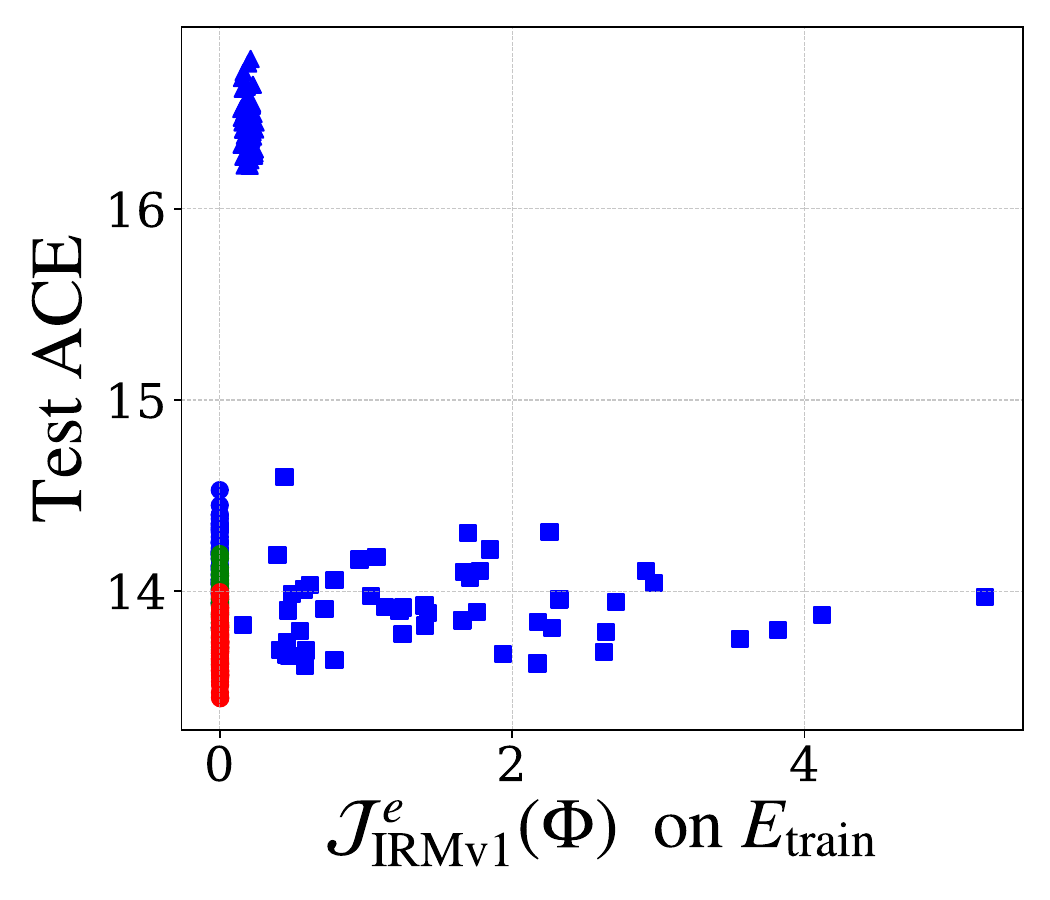}
    \label{fig:cobj_grad8}
  \end{subfigure}%
  \vspace{-5mm}
  \begin{center}
    \begin{tabular}{clclcl}
      \tikz \draw[blue,fill=blue] (0,0) circle (0.5ex); & IRMv1&
      \tikz \draw[blue,fill=blue] (0,0) rectangle (1ex,1ex); & BIRM&
      \tikz \draw[blue,fill=blue](0,0) -- (1ex,0) -- (0.5ex,0.86ex) -- cycle; & BLO\\
      \tikz \draw[matplotlibgreen,fill=matplotlibgreen] (0,0) circle (0.5ex); & v-IRMv1 (Ours)&
      \tikz \draw[red,fill=red](0,0)  circle (0.5ex); & mm-IRMv1 (Ours)\\
    \end{tabular}
  \end{center}

  \caption{ The relationship between the IRMv1 penalty values in the training environment and the
    corresponding test evaluation metrics (from left to right: accuracy, ECE, and ACE) on
    CMNIST. Each point represents the values recorded at each epoch during the last 50 epochs for
    each method. \TODO{Add analysis} }
  \label{fig:all_penalty}
  \vspace{-2mm}
\end{figure}


\end{document}